\documentclass{article}
\usepackage{supporting_files/arxiv}
\usepackage{silence}
\usepackage[utf8]{inputenc}
\usepackage[T1]{fontenc}
\usepackage{url}
\usepackage{microtype}
\usepackage[dvipsnames]{xcolor}
\usepackage{comment}
\usepackage[title]{appendix}

\usepackage{graphicx}
\usepackage{wrapfig}
\usepackage{subcaption}
\usepackage{float}
\usepackage[font=small,labelfont=bf]{caption}
\usepackage[normalem]{ulem}
\usepackage{placeins}
\usepackage{booktabs}
\usepackage{multicol}
\usepackage{multirow}
\usepackage{tablefootnote}
\usepackage{enumitem}
\usepackage{chngcntr}   

\usepackage{amsmath}
\usepackage{amsfonts}
\usepackage{amssymb}
\usepackage{mathtools}
\usepackage{mathrsfs}
\usepackage{nicefrac}
\usepackage{amsthm}
\usepackage{thmtools}
\usepackage{thm-restate}
\usepackage[ruled]{algorithm2e}
\usepackage{algpseudocode}
\usepackage{array, makecell}
\usepackage{centernot}

\usepackage{amsmath,amsfonts,bm}

\def\eqref#1{equation~\ref{#1}}
\def\Eqref#1{Equation~\ref{#1}}

\def\1{\bm{1}}

\DeclareMathAlphabet{\mathsfit}{\encodingdefault}{\sfdefault}{m}{sl}
\SetMathAlphabet{\mathsfit}{bold}{\encodingdefault}{\sfdefault}{bx}{n}

\DeclareMathOperator*{\argmax}{arg\,max}
\DeclareMathOperator*{\argmin}{arg\,min}

\declaretheorem{theorem}

\declaretheorem[numberwithin=theorem]{corollary}

\usepackage[sort]{natbib}
\usepackage{hyperref}
\usepackage[nameinlink,capitalize,noabbrev]{cleveref}
\definecolor{mydarkblue}{rgb}{0,0.08,0.45}
\hypersetup{colorlinks=true,
    linkcolor=mydarkblue,
    citecolor=mydarkblue,
    filecolor=mydarkblue,
    urlcolor=mydarkblue
}
\crefname{tcb@cnt@definitionbox}{Definition}{Definitions}

\usepackage[most]{tcolorbox}
\newtcolorbox[auto counter]{notebox}[2][]
{
  center,
  breakable,
  width = \linewidth,
  colframe = gray!25,
  colback  = gray!10,
  coltitle = black,  
  title    = \textbf{Note \thetcbcounter. #2},
  #1,
}
\newtcolorbox[auto counter]{definitionbox}[2][]{
  center,
  unbreakable,
  width = \linewidth,
  colframe = teal!30!white,
  colback = teal!3,
  coltitle = black,
  title = \textbf{Definition \thetcbcounter. #2},
  #1,
}

\usepackage[draft]{todonotes}   

\title{Defining and Generating Compositional Structure}
\renewcommand{\undertitle}{} 
\renewcommand{\headeright}{Preprint} 
\renewcommand{\shorttitle}{Defining and Generating Compositional Structure} 

\author{
    Eric Elmoznino$^{*,1,2}$, 
    Tejas Kasetty$^{1,2}$,
    Blaise Agüera y Arcas$^{3}$,
    Yoshua Bengio$^{1,2,4}$,
    Guillaume Lajoie$^{1,2,3}$ \\\\
    $^1$Mila -- Quebec AI Institute,
    $^2$Université de Montréal,
    $^3$Google Paradigms of Intelligence Team,
    $^4$LawZero
    \vphantom{
    \thanks{Correspondence to: \texttt{eric.elmoznino@gmail.com} \\
    Code is provided at \url{https://github.com/EricElmoznino/defining_and_generating_compositional_structure}}
    }
}

\hypersetup{
pdftitle={Defining and Generating Compositional Structure},
pdfsubject={deep-learning},
pdfauthor={Eric Elmoznino},
pdfkeywords={compositionality, complexity, structure, compression, machine learning},
}

\begin{document}
\maketitle

\begin{abstract}
    Compositional structure is pervasive in nature and across scales: atoms combine into molecules, molecules into cells, and cells into organisms, while further compositional hierarchies appear within each level. More broadly, compositionality --- the reuse and recombination of parts into novel structures --- is a fundamental source of interesting complexity in both natural and artificial systems. Yet despite its importance, and despite our apparent ease in recognizing it, mathematical accounts capturing what we intuitively mean by compositional structure remain limited. Here, we fill this gap by formalizing a simple idea: compositional structure emerges when parts, considered together, reveal shared structure that is not captured when those parts are modeled separately. We express this idea using algorithmic information theory --- the mathematics of information and compression --- which provides a natural way to identify structure in arbitrary data without specifying in advance what that structure must look like. Applied across scales, our definition informs where and how compositional structure emerges within an object. Using tractable approximations, we then invert the framework to generate compositional objects: we specify the desired form of their compositional structure without prescribing their particular content, and optimize directly over strings and images. The resulting objects exhibit readily recognizable hierarchical reuse and part–whole organization, suggesting that the definition captures compositional structure rather than merely complexity, regularity, or shared information. Finally, we discuss implications for machine learning, including a new perspective on depth in neural architectures and principles for constructing compositional data curricula capable of supporting open-ended learning.
\end{abstract}

\section{Introduction}

\textit{``Each portion of matter can be conceived as a garden full of plants, and as a pond full of fish. But each branch of the plant, each member of the animal, each drop of its humors, is also such a garden or such a pond.''} \\
\hspace*{\fill} --- Gottfried Wilhelm Leibniz, \citep{leibniz1898monadology}

Compositional structure is everywhere, and at every scale. Atoms combine into molecules, molecules into cells, cells into organs, and organs into organisms. Each of these levels in turn contains further rich hierarchies; cells and organs, too, are intricately organized assemblies of interacting parts. Compositionality applies to more abstract spaces as well: music develops motifs into phrases and movements; programs build modules from modules and functions from functions; concepts are built from other concepts; technologies are built from more granular technologies \citep{arthur2009nature}; scientific theories inherit and reorganize earlier ideas; and a well-designed curriculum continually recombines prior knowledge into more sophisticated skills. Across all these examples, complexity does not arise from an undifferentiated collection of elements, but from parts being reused and recombined into new structures. In this broad sense, compositionality is a signature of both life and intelligence, and arguably one of the principal sources of interesting complexity in the universe \citep{arcas2025intelligence}.

Despite how readily we recognize such examples, it remains difficult to say precisely what makes them compositional. Familiar descriptions --- that the whole is more than the sum of its parts, or that the meaning of a complex expression depends on its structure and the meanings of its constituents --- are suggestive but do not by themselves provide a general quantitative measure \citep{pagin2010compositionality}. More formal approaches in machine learning typically begin with a prescribed set of semantic constituents and a rule for reconstructing the whole from them. Some define these constituents using known generative factors or object slots \citep{wiedemer2023compositional,wiedemer2024provable}, while others restrict their combination to a particular operation, such as a linear map or a prescribed hierarchy \citep{trager2023linear,andreas2019measuring}. These approaches are useful when the parts and their modes of combination are already known, but they make compositionality relative to externally supplied semantics or a restricted form of composition \citep{elmoznino2024complexity}. They therefore do not fully answer the broader question: for an arbitrary object, what intrinsic compositional structure emerges across its parts, how much of it is present, and how is it organized across scales?

Our starting point is that compositional structure emerges when parts, considered together, reveal shared structure that is not captured when those parts are modeled separately. Here, a model represents the structure in an object: the regularities and patterns that provide a shorter explanation of the data. If the model of the whole contains something genuinely new beyond the models of the parts, then those parts share some pattern or organization that becomes apparent only through their combination. This is what we call \emph{novel compositional structure}. A cell viewed in isolation reveals its internal organization; a collection of cells viewed together may additionally reveal a tissue, with spatial arrangements and collective mechanisms absent from any one cell. The same principle can be applied recursively: tissues form organs, motifs form phrases, and functions form higher-level functions. Repeating the comparison at different locations and scales produces a profile of where an object's compositional structure emerges. \cref{sec:intuition} develops this idea and tests it against several intuitive cases.

To make this precise, we use algorithmic information theory: a mathematics of information expressed through description length and compression. The shortest two-part description of an object separates a model, which captures its compressible structure, from the residual information needed to specify the object given that model. We compare these structure-bearing models across parts and wholes, measuring how much additional information is needed to describe the model of the whole from the models of its parts. This lets us speak rigorously about parts, structure, and novelty without specifying in advance what the structure must be. The distinction is between \emph{form} and \emph{content}: our quantities describe where and how structure is composed while remaining agnostic to the particular structures or patterns involved. \cref{sec:defining} formalizes the definition, extends it to multiscale compositional structure profiles, and develops tractable approximations.

A quantitative definition also lets us look at the problem in reverse: instead of measuring the compositional structure of a given object, we can specify a desired profile and search for objects that exhibit it. We demonstrate this in \cref{sec:generating} by optimizing directly over strings and images. Without prescribing their particular content, the objective produces strings with nested patterns and images with coherent, hierarchical visual organization reminiscent of abstract art. These experiments serve both as a method for generating structured data and as a concrete test of the definition. The generated objects display recognizable reuse, hierarchy, and part--whole organization rather than mere randomness or repetition.

Finally, we consider what this perspective suggests for machine learning. Natural data often builds structure on top of structure across many scales, matching the inductive bias of deep architectures whose successive layers build increasingly abstract features from earlier ones. Compositional structure can also grow through time: new data may reuse, extend, or reorganize what came before. This suggests principles not only for choosing architectures, but for ordering experience into curricula in which each stage builds maximally on prior structure, and for gathering new data that continually expands what a learner can represent. We close by developing these connections between depth, curricula, and open-ended learning to our definition of compositional structure in \cref{sec:discussion}.

\section{Intuition for compositional structure}
\label{sec:intuition}

Before introducing the formal machinery, we build intuition through a running concrete example based on computer programs. We use computer programs because they are familiar and paradigmatic compositional objects: programmers explicitly define reusable parts and combine them to construct more complex ones, making the relevant structure unusually easy to see. A self-contained program can be separated conceptually into a \emph{library} of reusable functions, classes, and data structures, and the remaining code written using that library. When the library is designed to make the complete program as short as possible, it captures the program's reusable structure. A richly compositional program should then have a modular library containing a varied repertoire of components that are reused throughout the program and recombined to define higher-level components. A library with very few distinct components provides little material to compose. Conversely, a large collection of functions used only within separate portions of the code is better understood as several unrelated programs placed beside one another; size alone does not make a library compositional.

The same program can be organized around many different libraries, so inspecting an arbitrary implementation would make this intuition depend on choices made by its programmer. Compression provides a less arbitrary reference point: consider the library that yields the shortest overall description after accounting for both the library and the code expressed through it. Our strategy is to track how this optimal library changes as increasingly large parts of the program are considered together. If two snippets contain related logic, their joint description may justify a reusable abstraction that was not economical for either snippet alone. Compositionality then depends on whether the library of the whole contains something genuinely new: not new raw code, since the relevant code already occurred in the parts, but a new \emph{structural description} of what those parts share.

\looseness=1
More concretely, let $x_a$ and $x_b$ denote two program fragments and let $x=[x_a,x_b]$ denote their combination, with each expressed directly as code without factoring reusable pieces into a library. Write $m_{x_a}$, $m_{x_b}$, and $m_x$ for the optimal libraries that compress the corresponding fragments or whole. Suppose the same block of code occurs once in each fragment and is long enough that extracting it into a function becomes worthwhile only after the fragments are joined. The function need not appear in $m_{x_a}$ or $m_{x_b}$, since defining it for a single use would not shorten either fragment individually. For the whole, however, extracting the shared block into a function can shorten the total program. The library $m_x$ has grown because a common part becomes reusable structure only at the level of the whole. The same phenomenon can involve refactoring rather than simply adding a function. For example, one fragment may justify a matrix-multiplication routine while another justifies a vector-norm routine; considered together, the shortest library may redefine the norm in terms of the existing multiplication routine. In both cases, combining the fragments changes how their structure is optimally represented and reveals a novel interaction between them.

\looseness=1
Two contrasting cases clarify why this is more specific than either reuse or structural complexity alone. First, $x_b$ might simply provide more uses of a library that was already optimal for $x_a$. The whole is then compressed using the same library as the parts: there is shared structure and a compression benefit, but no new abstraction or refactoring is required. Second, $x_a$ and $x_b$ might be entirely unrelated programs. The optimal library of the whole may be larger than either part-library, but only because it places their independent contents side by side. Given $m_{x_a}$ and $m_{x_b}$, constructing $m_x$ requires essentially nothing novel. Both of these cases differ from the compositional one: in the first, all the structure needed to compress both parts is already encoded in one of their libraries; in the second, there is no shared structure to exploit. Compositional structure lies in a different space, where the best joint library cannot be obtained merely by selecting or merging the libraries of the parts.

\looseness=1
This comparison can be applied recursively. A program may reveal one abstraction when short fragments are joined, another when larger segments of code are combined, and still others when those segments form the complete program. Recording where the optimal library grows or is reorganized therefore describes not just whether the program is compositional, but where its compositional structure emerges across locations and scales. The comparison at each join is local: it asks what additional library structure emerges when those particular parts are combined. Structure already captured within a smaller segment is therefore attributed to the earlier joins inside that segment, while later joins capture only the further structure introduced by combining it with other segments.

\looseness=1
Programs and libraries provide a concrete example of a more general idea. For an arbitrary object, a \emph{model} plays the role of the library: it captures regularities that allow the object to be described more concisely, while the remaining description specifies details not captured by the model. The object may represent music, an image, an \textit{iid} dataset, or a nonstationary data stream. In every case, we ask the same question of a chosen division into parts: does the best model of the whole contain structural information that cannot be recovered from the best models of the parts? The next section makes ``best model'' and ``cannot be recovered'' precise using algorithmic information theory. Importantly, the claim is not that the raw whole contains information absent from its parts: by construction, it is simply their concatenation. Rather, combining the parts can change the best account of their structure. It is in this precise sense that the description of a compositional whole is more than the sum of the descriptions of its parts.

\section{Defining compositional structure}
\label{sec:defining}

The preceding intuition suggests that compositional structure can be identified by comparing the structure-bearing descriptions of parts and wholes. We now make this idea precise using algorithmic information theory. We first use optimal two-part codes \citep{grunwald2007minimum} to define the model that captures an object's compressible structure; we then measure the novel structure that emerges at a composition step, apply this comparison recursively to obtain multiscale compositional structure profiles, and describe how the resulting quantities can be approximated in practice.

\subsection{Preliminaries: algorithmic information theory}

\paragraph{Kolmogorov complexity.}

Kolmogorov complexity \citep{li2008introduction,kolmogorov1965three,fortnow2000kolmogorov} is a notion of information quantity that serves as the foundations for algorithmic information theory. Intuitively, the Kolmogorov complexity of an object $x$, denoted $K(x)$, is the length of the shortest program (in some programming language) that outputs $x$. A related notion is the conditional Kolmogorov complexity of $x$ given another object $y$, denoted $K(x \mid y)$, which is the length of the shortest program that takes $y$ as input and outputs $x$. Kolmogorov complexity has many intuitive properties as a measure of information quantity that generalize equivalent properties from Shannon information theory \citep{grunwald2003kolmogorov}, and is deeply rooted in the idea of \emph{compression}.

Formally, throughout this paper $K$ denotes \emph{prefix} Kolmogorov complexity relative to a fixed universal prefix Turing machine $U$. A prefix Turing machine is one in which valid programs are self-delimiting; they are commonly-used in algorithmic information theory for technical convenience. For a finite string $x$, $K(x)$ is the length $l(r)$ (in bits) of the \emph{shortest} binary program $r$ that prints $x$ and halts:
\begin{align}
    K(x) = \min_r \{ l(r) : U(r) = x, r \in \{0,1\}^* \} ,
\end{align}
and given a separate string $y$, the conditional Kolmogorov complexity $K(x \mid y)$ is defined analogously as:
\begin{align}
K(x \mid y) = \min_r \{ l(r) : U(r,y) = x, r \in \{0,1\}^* \} .
\end{align}
For tuples, let $\langle x,y\rangle$ denote a fixed computable self-delimiting encoding of $x$ and $y$, and use $K(x,y)$ as shorthand for $K(\langle x,y\rangle)$; conditioning on several objects is interpreted similarly. The particular choices of universal machine and tuple encoding affect these quantities by at most an additive constant independent of the objects. In general, $K(\cdot)$ is incomputable. However, the quantity can still be useful and approximated.

\begin{notebox}{Suppression of lower-order terms}
    Kolmogorov-complexity relationships generally hold only up to small additive terms. Changing the universal machine or a fixed encoding contributes an $\mathcal{O}(1)$ constant, while recording lengths, boundaries, or other auxiliary information can contribute $\mathcal{O}(\log n)$ bits for objects of length $n$. Throughout the paper, we suppress these terms and write the corresponding equalities and inequalities directly. This convention keeps the description-length arguments transparent without changing their qualitative content. Accordingly, statements involving zeros, strict inequalities, or exact ties should be understood up to these suppressed terms: the relevant differences are assumed to be large enough that coding overheads cannot change the comparison.
\end{notebox}

\paragraph{Structure and Occam's razor.}

Although Kolmogorov complexity provides a canonical notion of information quantity, it does not separate structural complexity from randomness. A regular string has low complexity because it can be generated by a short program, whereas a random string has high complexity because no substantially shorter description exists than the string itself. However, the latter form of complexity does not reflect elaborate internal organization; it reflects the absence of compressible structure. In this sense, $K(x)$ entangles two distinct contributions to description length: one arising from meaningful structure, and one arising from irreducible randomness. This motivates finer notions that aim to measure structure rather than mere incompressibility.

This distinction between structural and random information was important to Kolmogorov himself, who quantified it through a new metric \citep{kolmogorov1973tallinn} that was later generalized under various formulations and termed \emph{sophistication} or \emph{meaningful information} \citep{grunwald2008algorithmic,gacs2001algorithmic,vitanyi2006meaningful,koppel1989program}. Intuitively, measuring structural information involves compressing an object using two-part coding schemes that follow the minimum description length (MDL) principle \citep{grunwald2007minimum} and Occam's razor.

Let $\mathcal{M}$ be a model class, and denote the two-part code length of $x$ under a model $m \in \mathcal{M}$ as:
\begin{align}
    L_m(x) = K(m) + l_m(x) ,
\end{align}
where $l_m(x)$ is the model-to-data code length (i.e., the cost of describing the data given the model). The form of $l_m(x)$ depends on the model class under consideration. For instance, in the case of computable probability distributions, $l_m(x) = -\log_2 m(x)$, where $m(x)$ is the probability mass that $m$ assigns to $x$.

We define the set of minimum description length models by:
\begin{align}
    \mathcal{M}^*(x) := \argmin_{m\in\mathcal{M}} L_m(x) .
\end{align}

Assuming there is a unique model in $\mathcal{M}^*(x)$ with minimal Kolmogorov complexity, we define the \textbf{Occam’s razor model} of $x$, denoted $m_x$, by:
\begin{align}
    \label{eq:occam}
    m_x := \argmin_{m\in\mathcal{M}^*(x)} K(m) .
\end{align}
The model $m_x$ relates to what has been called the \emph{minimal algorithmic sufficient statistic} of $x$ in prior literature \citep{grunwald2008algorithmic,gacs2001algorithmic,vereshchagin2016algorithmic}. $K(m_x)$ captures the structural information in $x$ that can be exploited for compression within the model class $\mathcal{M}$, while the residual term $l_{m_x}(x)$ represents the remaining incompressible, unstructured, and random information in $x$ relative to that model class. For sufficiently expressive model classes, one has $L_{m_x}(x) = K(m_x) + l_{m_x}(x) = K(x)$, so that $m_x$ decomposes Kolmogorov complexity into structured and unstructured components.

\subsection{Formal definition}

\paragraph{Novel compositional structure.}

Using the notion of an Occam's razor model defined in \cref{eq:occam}, we provide our formal definition for novel compositional structure.

\begin{definitionbox}[label=def:compositionality]{Novel compositional structure}
    Let $x=[x_a,x_b]$ denote a decomposition of a string $x$ into two substrings $x_a$ and $x_b$, where $[x_a,x_b]$ denotes concatenation. Let $m_\cdot$ denote an Occam's razor model as defined in \cref{eq:occam}, and assume that it is uniquely defined for each of $x$, $x_a$, and $x_b$. We define the amount of \emph{novel compositional structure} in $x$ relative to this decomposition as:
    \begin{align}
        C(x_a,x_b \to x) := K(m_x \mid m_{x_a},m_{x_b}) .
    \end{align}
\end{definitionbox}

This quantity measures the additional information required to describe the model of the whole string, $m_x$, given the models of the parts, $m_{x_a}$ and $m_{x_b}$. In particular, $C(x_a,x_b \to x) \approx 0$ in cases where the structure of the whole is already recoverable from the structures of the parts. For example, this holds when $m_x$ is identical to one of $m_{x_a}$ or $m_{x_b}$, or when $m_x$ is a trivial combination of the pair $(m_{x_a},m_{x_b})$ with only minor overhead (e.g., $m_x$ compresses $x$ by using $m_{x_a}$ to compress $x_a$ and using $m_{x_b}$ to compress $x_b$).

Thus, $C(x_a,x_b \to x)$ is intended to quantify structure present in the whole that is not recoverable from the models of the parts alone. Intuitively, $C(x_a,x_b \to x)$ is large when jointly considering $x_a$ and $x_b$ reveals a novel compression strategy that does not arise when they are modeled separately, such as when a pattern repeats across the boundary between $x_a$ and $x_b$ but not within either substring individually. More broadly, it measures the extent to which the structure of the whole is \emph{different} from the structures of the parts.

Note that \cref{def:compositionality} measures only \emph{novel} compositional structure from the decomposition of $x$ into $x_a$ and $x_b$; it does not quantify the compositional structure that is already present in the parts themselves.

\paragraph{Compositional structure profiles.}

The quantity $C(x_a,x_b \to x)$ measures novel compositional structure relative to a single decomposition of a string. More generally, if $x=[x_1,\dots,x_k]$, we define:
\begin{align}
    C(x_1,\dots,x_k \to x) := K(m_x \mid m_{x_1},\dots,m_{x_k}) .
\end{align}
This definition can be applied recursively with respect to any chosen decomposition strategy, which allows us to characterize a string's \textit{compositional structure profile}.

\begin{definitionbox}[label=def:compositionality_profile]{Compositional structure profile}
    Let $T$ be a rooted decomposition tree for a string $x$, where each node $v$ is associated with a substring $x_v$ and children $v_1,\dots,v_{k_v}$ satisfying $x_v = [x_{v_1},\dots,x_{v_{k_v}}]$. For each such node and children, we define its local compositionality by:
    \begin{align}
        C_{T,x}(v) := C(x_{v_1},\dots,x_{v_{k_v}} \to x_v)
                    = K(m_{x_v}\mid m_{x_{v_1}},\dots,m_{x_{v_k}}) .
    \end{align}
    We define a string $x$'s \emph{compositional structure profile} with respect to a rooted decomposition tree $T$ as:
    \begin{align}
        \label{eq:compositionality_profile}
        \mathcal{P}_T(x)
        :=
        \left(C_{T,x}(v)\right)_{v\in \mathrm{Int}(T)}
        \in \mathbb{R}_{\ge 0}^{\mathrm{Int}(T)} ,
    \end{align}
    where $\mathrm{Int}(T)$ denotes the set of internal nodes of $T$.
\end{definitionbox}

Intuitively, \cref{def:compositionality_profile} says that we can consider a hierarchical decomposition of a string into a tree and record how much novel compositional structure appears at each composition step in that tree's hierarchy. For example, one may choose $T$ to be a balanced binary tree obtained by recursively splitting $x$ into two halves until a prescribed minimum substring length is reached. In that case, $C_{T,x}(v)$ is measured at every join operation in the recursive decomposition, and $\mathcal{P}_T(x)$ captures the emergence of novel compositional structure in $x$ at multiple scales and locations.

\begin{figure}[ht]
    \centering
    \includegraphics[width=\linewidth]{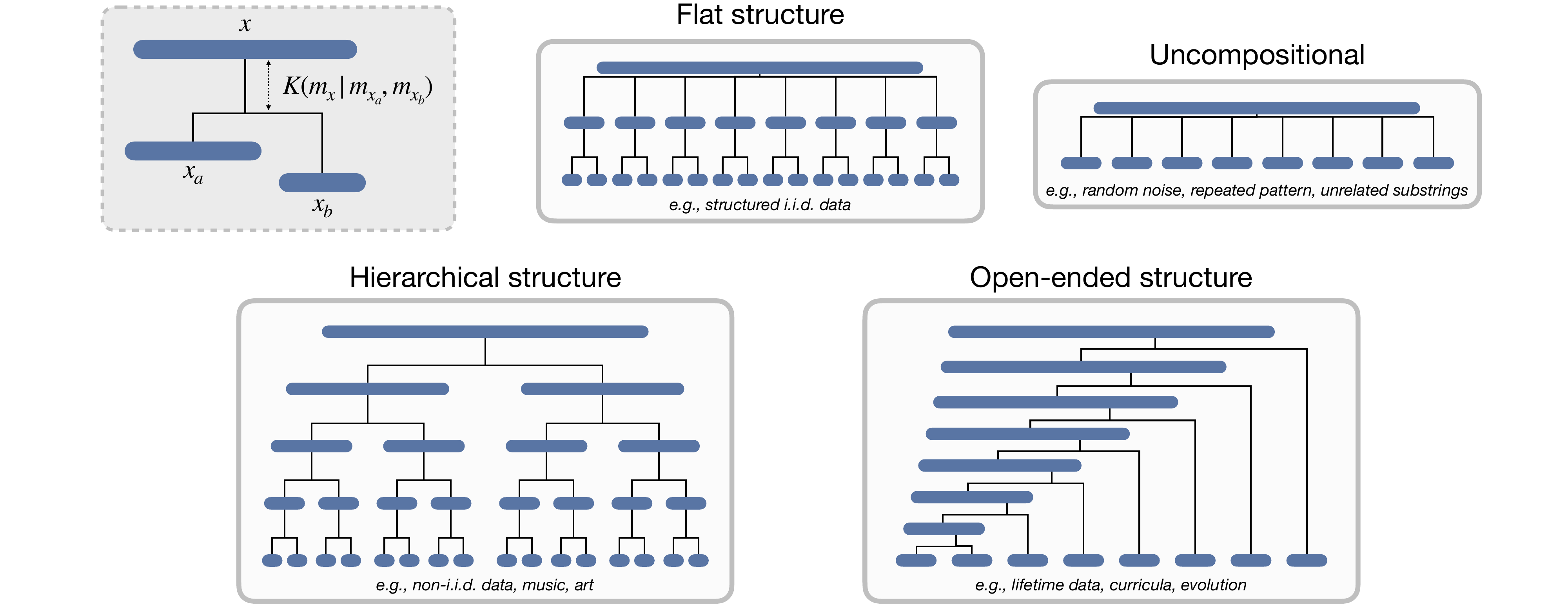}
    \caption{\textbf{Visual illustration of compositional structure profiles.} A profile records where novel structure emerges along a hierarchical decomposition of an object. At each merge, the vertical increment represents the additional information required to describe the model of the whole from the models of its parts, $C(x_a, x_b \to x) := K(m_x \mid m_{x_a},m_{x_b})$ (top left). Different distributions of these increments produce characteristic signatures: negligible novel structure for uncompositional objects, structure confined to bounded scales for structured \textit{iid} data, structure distributed across nested scales for hierarchical objects, and continually accumulating structure for open-ended processes.}
    \label{fig:compositional_structure_profiles}
\end{figure}

\cref{fig:compositional_structure_profiles} visually illustrates how compositional structure profiles provide signatures not only of how much structure an object contains, but of how that structure is organized. Random or trivially regular objects exhibit little novel structure throughout the decomposition. For structured \textit{iid} data, novel structure appears only up to a characteristic scale: once enough observations determine the shared data-generating process, additional samples introduce no new model structure since the arrangement of the samples is arbitrary and random. Hierarchical objects such as music and art instead exhibit novel structure across many nested scales, as smaller patterns compose into progressively larger ones. Under a left-to-right decomposition, open-ended processes such as curricula, lifetime experience, or evolution produce a different signature, with each new segment extending or refactoring the structure accumulated so far. These qualitatively distinct profiles make the organization and growth of an object's structure apparent.

An image dataset makes the connection between hierarchical compositional structure and scale concrete. A single image of a cat contains many edge segments, textures, and simple shapes, so these recurring structures may already justify their place in its optimal model; with only one cat, however, introducing a reusable concept of a cat may not yet repay its description length cost. Across many images, cats recur often enough for the joint model to represent them as compositions of lower-level structures; with still more data, rarer regularities can justify further distinctions such as breeds or recurring relations between cats and indoor or outdoor scenes. Larger datasets therefore expose progressively higher levels of the hierarchy, since rarer combinations involving more parts require more observations before they improve compression. The compositional structure profile records the scales at which these new abstractions become worthwhile, and for non-stationary data additionally records the locations.

\paragraph{The choice of decomposition.}

Up to this point, \cref{def:compositionality_profile} has defined a compositional structure profile in a way that depends on the choice of decomposition tree $T$. However, intuitions about compositionality strongly suggest that a given object has an \emph{intrinsic}, or ``true'' decomposition. We think of a compositional object has having parts that combine to form wholes in a particular way, and capturing those intrinsic hierarchies is an important part of defining compositional structure. We leave this as an open problem for future work, which bears on several important questions. For instance, if one were to use a \emph{different} decomposition from this intrinsic one, would the resulting compositionality structure profile still be meaningful? Do there exist objects which are intuitively non-compositional, but erroneously appear as if they were under the wrong decomposition? For an object that was originally constructed through a series of compositions, as is the case in the evolution of life \citep{arcas2025intelligence} and technology \citep{arthur2009nature}, would the intrinsic decomposition recover this true history?

Although we leave these questions unanswered in the present work, we conjecture that an object's intrinsic decomposition is the one that maximizes the sum of scores $C(x_{v_1},\dots,x_{v_{k_v}} \to x_v)$ across the tree (i.e., the sum of \cref{eq:compositionality_profile}). We justify this more extensively in \cref{sec:maximal_decomposition}, but the intuition is that this would result in joins that maximally privilege the emergence of shared parts throughout the decomposition, capturing the object's natural hierarchy.

\subsection{Approximating compositional structure}
\label{sec:approximating_compositional_structure}

The preceding definitions cannot be computed exactly. In addition to the uncomputability of the Kolmogorov complexities used to describe models and measure conditional novelty, identifying an Occam's razor model requires searching a potentially infinite class of computable models. Even if progressively better descriptions are found, there is generally no way to certify that no unexamined model would yield a shorter two-part code.

In practice, we replace these quantities with empirical coding and search procedures. Let $f$ denote an empirical two-part coding scheme, together with an associated model class and fitting procedure, that assigns the code length:
\begin{align}
    L_m^f(x)
    :=
    L_{\mathrm{model}}^f(m)
    +
    L_{\mathrm{data}}^f(x\mid m)
\end{align}
and returns an empirical MDL model $m_x^f$. Let $g$ denote an empirical conditional coding scheme for describing one fitted model given others. We then define the empirical compositionality of $x=[x_1,\dots,x_k]$ by:
\begin{align}
    C^{f,g}(x_1,\dots,x_k\to x)
    :=
    L_{\mathrm{cond}}^g
    \bigl(
        m_x^f
        \mid
        m_{x_1}^f,\dots,m_{x_k}^f
    \bigr).
\end{align}
Thus, $f$ approximates both the description lengths and the search involved in constructing the models, while $g$ measures the additional cost of describing the fitted model of the whole from those of its parts. The resulting quantity is an operational, coding-strategy-dependent notion of compositionality, rather than a formal bound on the corresponding algorithmic-information quantity. Further details are provided in \cref{sec:approximating_compositional_structure_details}.

\subsection{Relationship to other measures of compositionality}

The literature on compositionality has a long history spanning many fields, so we do not aim to exhaustively survey it here. Instead, we focus on formal definitions of compositionality that have been proposed (at times under different names), with a particular emphasis on approaches leveraging algorithmic or Shannon information theory.

\paragraph{Algorithmic mutual information.}

For an object $x = [x_a, x_b]$, algorithmic mutual information (AMI) quantifies the degree to which the whole is more compressible than the parts:
\begin{align*}
    I(x_a : x_b) = K(x_a) + K(x_b) - K(x_a, x_b) .
\end{align*}
AMI has been proposed as a measure for what we here call novel compositional structure \citep{arcas2025intelligence,schmidhuber1997low}, and indeed it bears many similarities to the definition we have proposed. In particular, under very mild assumptions it can be shown that:
\begin{align*}
    C(x_a, x_b \to x) > 0 \implies I(x_a : x_b) > 0 ,
\end{align*}
but that the reverse does not hold (proof in \cref{sec:proofs_ami}). Intuitively, two objects can share AMI for one of two reasons: (i) they share existing structure in $m_{x_a}$ and $m_{x_b}$ that need not be redundantly encoded in $m_x$, or (ii) they share novel structure in $m_x$ that only appears when they are considered jointly and does not exist in either $m_{x_a}$ or $m_{x_b}$ alone. $C(x_a, x_b \to x)$ measures only the latter case, disentangling the emergence of novel compositional structure from simply the accumulation of more data with the same structure.


\paragraph{Partial information decomposition.}

Partial information decomposition (PID) separates the information that two sources $X_a$ and $X_b$ provide about a target $Y$ into redundant, unique, and synergistic components \citep{williams2010nonnegative}:
\begin{align*}
    I(X_a,X_b;Y)
    &=
    \mathrm{Red}(X_a,X_b;Y)
    + \mathrm{Unq}_a(X_a,X_b;Y) \\
    &\quad
    + \mathrm{Unq}_b(X_a,X_b;Y)
    + \mathrm{Syn}(X_a,X_b;Y) .
\end{align*}
The synergistic term measures information about $Y$ that is available from $(X_a,X_b)$ jointly but from neither source individually. It is therefore the closest Shannon-theoretic analogue of $C(x_a,x_b\to x)$: both notions aim to isolate information that becomes available only when the parts are considered jointly.

The two quantities nevertheless address different settings. PID is a property of a joint distribution over sources and a target, whereas $C(x_a,x_b\to x)$ is an algorithmic property of the internal structure of an individual object. This distinction is particularly important in the intrinsic part--whole setting. If the PID target is chosen to be the whole, $Y=(X_a,X_b)$, then we trivially have that:
\begin{align*}
    \mathrm{Red}(X_a,X_b;(X_a,X_b))
    &= I(X_a;X_b), \\
    \mathrm{Syn}(X_a,X_b;(X_a,X_b))
    &= 0 .
\end{align*}
Consequently, such a PID cannot detect structure that emerges when the sources are combined into the target. PID-based measures of linguistic compositionality instead introduce an external target, such as a word form generated from multiple semantic features \citep{socolof2022measuring}. These measures characterize the compositionality of a mapping between variables, rather than the compositional structure intrinsic to an object.

Finally, although PID requires its components to sum to the corresponding mutual information quantities, it does not uniquely determine how information should be allocated among the redundant, unique, and synergistic terms. Consequently, there is no generally accepted decomposition, and existing proposals satisfy different, sometimes incompatible, auxiliary axioms \citep{harder2013bivariate,matthias2025novel}. Our definition does not attempt a complete information decomposition and therefore avoids selecting among these alternatives: it directly measures the conditional description length $K(m_x\mid m_{x_a},m_{x_b})$ of the whole's structural model, thereby quantifying structural information that emerges only when the parts are considered jointly --- an algorithmic analogue of synergistic information.

\paragraph{Emergence and Kolmogorov's structure function.}

Emergence is commonly understood as the appearance of system-level properties that are novel relative to those of the system's constituents, making it conceptually close to compositionality. \citet{bedard2022algorithmic} propose an algorithmic formalization based on Kolmogorov's structure function \citep{kolmogorov1973tallinn,vereshchagin2004kolmogorov}. Intuitively, the structure function considers increasingly large complexity budgets for describing a model of an object and, at each budget, identifies the model $m$ that minimizes the remaining model-to-data codelength $l_m(x)$. A substantial drop indicates that permitting a more complex model reveals a new explanatory strategy that shortens the overall two-part description. Their definition associates emergence with the presence of several such drops, corresponding to multiple useful models of increasing complexity. For example, a physical system may admit a relatively simple thermodynamic description in terms of macroscopic variables, while larger model budgets permit more detailed descriptions involving microscopic dynamics and initial conditions.

However, the complexity of a model is not itself a scale in a decomposition of the observed object. Consequently, successive drops need not correspond to structure within progressively larger collections of parts: they may instead reflect different explanatory strategies for the same data. The framework therefore characterizes emergence across a hierarchy of model complexities, but does not specifically isolate structure that emerges when constituent objects are combined as $C(x_a,x_b\to x)$ does. The two approaches are thus complementary: the structure function captures the plurality of emergent explanations available at different levels of model complexity, whereas our definition quantifies emergent structure arising at particular composition steps.

\paragraph{Formal grammars and compositional semantics.}

In formal semantics and programming-language theory, compositionality is conventionally defined relative to a specified syntax and interpretation: the meaning of a compound expression is determined by the meanings of its constituents and the rules by which they are combined \citep{pagin2010compositionality}. If $f$ is a syntactic composition rule and $M$ is an interpretation, this principle can be expressed as:
\begin{align*}
    M(f(a,b)) = F_f(M(a),M(b)) ,
\end{align*}
where $F_f$ is the corresponding operation in the semantic domain. Category-theoretic formulations express the same principle through structure-preserving maps, such as functors, between syntactic and semantic structures. These frameworks characterize compositionality only once the constituents, composition rules, and semantic domain have been specified; they do not by themselves quantify how much compositional structure an object intrinsically contains, nor do they identify that structure from arbitrary data.

Our definition addresses these complementary questions. A grammar or functional program may itself constitute the model $m_x$, in which case $C(x_a,x_b\to x)$ measures the additional information required to describe the grammar or library of the whole from those of its parts. In functional-programming terms, this resembles the cost of refactoring: combining $x_a$ and $x_b$ may reveal a reusable abstraction that was not justified by either part separately. In this way, $C(x_a,x_b\to x)$ is complementary to these more traditional views of compositionality: it can describe the emergence of compositional structure in a formal grammar.

\paragraph{Assembly theory.}

Assembly theory grew out of origins-of-life research as an attempt to differentiate living from non-living objects \citep{sharma2023assembly}. It treats a living object as something that exists in large numbers of copies and is built from simpler pieces through a sequence of allowed physical events. Once a piece has been built, it can be reused, so the object's \emph{assembly index} is the length of its shortest construction path. Assembly theory is therefore an MDL-based framework for characterizing living systems in which composition and hierarchy are central. In more technical ways, however, it differs substantially from our definition of compositionality. Among other things, our definition measures novel structure in the model of an individual object relative to the models of its parts, without requiring a physical construction space or observed copies. It also extends to more abstract forms of composition and distinguishes structure from random detail within an individual object.

\paragraph{Symbiogenesis.}

Symbiogenesis is the evolutionary idea that new kinds of organisms can arise when previously separate living systems join and become interdependent \citep{sagan1967origin}. The familiar example is the bacterial ancestor of mitochondria becoming part of a host cell, giving rise to eukaryotic cells. \citet{arcas2025intelligence} casts this as a general route to hierarchical complexity in which smaller replicators combine into a larger replicator, which can then combine with others; the total structure involves describing both the parts as well as how they fit together into the larger whole. Interestingly, in this computational account, the parts and whole carry instructions for their own reproduction --- i.e., they are \emph{self-encoding}. Our definition places no such requirement on an object's parts or whole; it measures the novel structure that appears when the parts are considered together, however it arose. Thus, symbiogenesis offers a plausible mechanism for the emergence of the structure that $C(x_a, x_b \to x)$ measures, and for its prevalence in the natural world.

\section{Generating compositional structure}
\label{sec:generating}

The previous section described a rigorous framework for quantifying compositionality and characterizing an object's compositional structure profile. We now turn to the inverse problem: given a pre-specified compositional structure profile, can we \emph{generate} objects that satisfy it?

For the remainder of this section, we consider restricted and tractable coding schemes $f$ and $g$ as described in \cref{sec:approximating_compositional_structure}, which define an empirical compositionality score
$
C^{f,g}(x_1,\dots,x_k\to x)
:=
L_{\mathrm{cond}}^g
\bigl(
    m_x^f
    \mid
    m_{x_1}^f,\dots,m_{x_k}^f
\bigr)
$
that we can compute.
For an object $x$, a particular decomposition tree $T$ gives us the empirical compositional structure profile:
\begin{align}
    \mathcal{P}_T^{f,g}(x) := \left(C_{T,x}^{f,g}(v)\right)_{v\in \mathrm{Int}(T)} \in \mathbb{R}_{\ge 0}^{\mathrm{Int}(T)} .
\end{align}
This profile is an indexed family of local compositionality values, one for each composition step in the tree. From this profile, we may define any scalar functional $s:\mathbb R_{\ge 0}^{\,\mathrm{Int}(T)}\to\mathbb R$, where the quantity $s\left(\mathcal{P}_T^{f,g}(x)\right)$ then defines a scalar summary of the compositional structure of $x$ under $T$. This summary quantity, if well behaved, can define a goal to be maximized by learning algorithms operating over the space of $x$'s.  

This allows us to generate objects with desired compositional structure profiles by defining the reward:
\begin{align}
    \label{eq:reward}
    R(x \ ; \ T, f, g, s) = s\left(\mathcal{P}_T^{f,g}(x)\right) .
\end{align}
and then using any optimization or sampling procedure to search for objects with high reward.
In sum, each of our below experiments is fully characterized by:
\begin{itemize}
    \item A domain $\mathcal{X}$ over which we can search for objects with high reward. Here, we consider two domains: symbolic strings and images.
    \item A pre-specified decomposition tree $T$, coding schemes $f$ and $g$, and scalar functional $s$ which jointly define the reward $R(x \ ; \ T, f, g, s)$.
    \item An optimization or search algorithm capable of discovering objects $x \in \mathcal{X}$ with high reward.
\end{itemize}

\subsection{Compositional strings}

We first consider a deliberately minimal domain: fixed-length symbolic strings with no externally specified semantics. Concretely, the domain is $\mathcal{X}=\{0,\dots,7\}^{128}$; strings of length $128$ with a vocabulary of $8$ symbols. Methods are briefly described below and illustrated in \cref{fig:strings}a; \cref{sec:experiment_details_strings} contains further details.

\begin{figure}[ht]
    \centering
    \includegraphics[width=\linewidth]{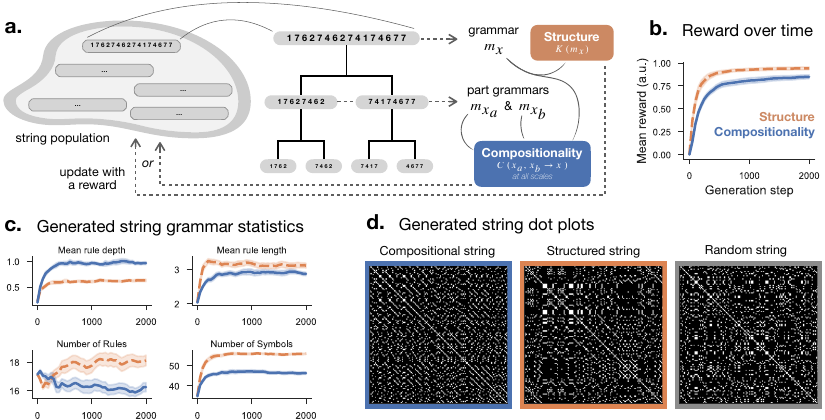}
    \caption{\textbf{Generating compositional strings.}
    \textbf{a.} We evolve a population of strings using either (i) a reward that maximizes multi-scale compositional structure using our definition (\textcolor{MidnightBlue}{\textbf{blue}}; \cref{eq:experiment_reward}) or (ii) an alternative reward that just maximizes structure without any pressure for compositionality (\textcolor{BurntOrange!50!RawSienna}{\textbf{orange}}; $K(m_x)$).
    \textbf{b.} Growth of the average population reward over time. Both the compositionality and structure rewards are successfully optimized. We plot in arbitrary units (a.u.) on the y-axis normalized based on the top score for each reward.
    \textbf{c.} Statistics of grammars fit to strings over time, when optimizing each reward. Grammars fit to strings generated by the compositionality reward are significantly deeper, with more nested levels of hierarchy, despite being shorter than those of strings generated from the structure reward in terms of rule length, number of rules, and number of symbols.
    \textbf{d.} Dot plots of the highest-reward string optimized using the compositionality and structure rewards, and for a random string. Strings generated from the compositionality reward have substantial repeated patterns, corresponding to contiguous diagonal lines off of the main diagonal. These patterns themselves contain repeated substrings, corresponding to shorter diagonal lines appearing within the same block of rows/columns. In contrast, strings generated from the structure reward exhibit more random segments, and when repeats do occur they tend to be longer and less hierarchical.
    \textbf{Seeds.} Shaded regions in panels (b, c) reflect the standard error of the mean across $10$ seeds. Panel (d) shows results for one representative seed only; dot plots for all seeds are provided in \cref{sec:experiment_details_strings}.}
    \label{fig:strings}
\end{figure}

\paragraph{MDL algorithm.}

For the two-part code $f$, we represent a string using a deterministic context-free grammar together with a residual sequence expressed through that grammar. Here, a grammar is simply a collection of named replacement rules: a rule assigns a name to a sequence symbols, which can also contain the names of previously defined rules. For example, the rules $A\to 01$ and $B\to AA$ allow $B$ to stand for $0101$. In standard terminology, the original symbols are \emph{terminals} and the rule names are \emph{non-terminals}. Because rules may invoke other rules, the grammar resembles a dictionary whose entries can themselves be defined from earlier entries, allowing repeated patterns to be represented hierarchically. The grammar serves as the model, while the residual specifies how its reusable components and any remaining terminal symbols are assembled into the complete string. We measure the two code components using an entropy code over all terminal symbols and rule names appearing in the grammar and residual. To fit a short grammar without introducing unjustified structure, we use an MDL-guided extension of RePair and MR-RePair \citep{larsson2000off,furuya2018mr}: it greedily replaces repeated maximal substrings with new rule names, but introduces a rule only when its estimated compression gain exceeds the cost of describing the rule itself.

\paragraph{Conditional model compression.}

At each join in a decomposition tree, two adjacent substrings are combined into a larger substring. We fit grammars separately to the two parts and to their combination, and use the conditional code $g$ to measure how much of the combined grammar cannot be recovered from the grammars fitted to the parts. Since rule names are arbitrary, we first relabel the rules in each part grammar to find a low-cost alignment with the combined grammar, using the nested structure of the rules to determine which ones correspond. Each rule in the combined grammar is then either encoded directly or by describing its differences from the best-matching aligned rule in either part grammar, whichever is shorter. Thus, carrying an existing rule from a part into the combined model costs little, whereas a new rule that becomes economical only after the parts are combined incurs a large cost.

\paragraph{Reward.}

Our reward is designed to favour strings that reveal novel grammar structure throughout the hierarchy, rather than concentrating it at a single scale. We therefore choose the decomposition tree $T$ to be a balanced binary tree: the complete string is recursively bisected into contiguous halves until reaching substrings of length four. Its internal nodes form five join scales, ranging from combinations of length-four substrings to the combination of the two length-64 halves of the complete string. At every join, we compute the conditional grammar cost described above and normalize it by the length of the combined substring. We then average these scores within each scale, so that scales containing many small joins do not dominate those containing a few large joins. Finally, we combine these scale averages using a power mean with $p=0.75$. Choosing $p<1$ places greater pressure on the weakest scales than an arithmetic mean would, encouraging compositional structure to emerge more evenly across the hierarchy.

Formally, let $\mathcal{V}_\ell$ be the joins at scale $\ell$, let $r_\ell(x)$ denote their mean size-normalized compositionality, and let $L$ denote the total number of scales. Then:
\begin{align}
    \label{eq:experiment_reward}
    r_\ell(x) &= \frac{1}{|\mathcal{V}_\ell|}\sum_{v\in\mathcal{V}_\ell}\frac{C_{T,x}^{f,g}(v)}{|x_v|}, \\
    R(x) &= \left(\frac{1}{L}\sum_{\ell=1}^{L} r_\ell(x)^p\right)^{1/p}, \quad p=0.75 .
\end{align}

\paragraph{Optimization.}

We maximize this reward with a single-objective use of NSGA-II \citep{deb2002fast}. Starting from $10{,}000$ uniformly random strings, each generation produces $5{,}000$ offspring through tournament selection and two-point crossover, followed by independently resetting each symbol with probability $2/128$. Duplicate strings are removed, and elitist selection retains the best population from parents and offspring.

\paragraph{Results.}

We ran the evolutionary optimization for $2{,}000$ generations across $10$ seeds. In addition to our compositionality reward in \cref{eq:experiment_reward}, we also compared to a structure-based reward that optimizes for $L_{\text{model}}^f(m_x^f)$: the empirical estimate of string $x$'s final structural complexity $K(m_x)$. This comparison serves as an empirical test of our claim that compositionality captures a specific \emph{kind} of structure based on the interactions between parts at multiple scales, rather than simply the total amount of structure. We used identical hyperparameters to optimize both rewards. \cref{fig:strings}b shows the time-course of the rewards, both of which steadily grow before reaching a plateau, indicating they were successfully optimized.

To demonstrate the properties of our compositionality reward, we analyze the nested repeated structure present within the strings by computing summary statistics of their grammars, shown in \cref{fig:strings}c. Here, we fit these grammars using the MR-RePair algorithm \citep{furuya2018mr}, which recursively creates a rule whenever a pattern repeats at least once. Compared to structure-optimized strings, we find that compositionality-optimized strings result in substantially \emph{deeper} grammars, where the depth of a grammar rule denotes the number of recursive non-terminal symbol expansions needed before the flattened rule contains only terminal symbols. For example, a grammar describing a string with purely ``flat'' structure would consist entirely of depth-$0$ rules that each expand to a sequence of terminal symbols. Despite this, structure-optimized strings still result in highly complex grammars; in fact, these grammars tend to have \emph{more} rules and longer rule expansions than those of the compositionality-optimized strings. This illustrates the core property of our compositionality definition. In general, an object can be structured in multiple ways that do not require a rich and interconnected hierarchy of patterns --- this structure can be flat (as in long but low-depth grammar rules) and independent across the object (as in grammar rules that are only used in distinct regions of the object).

We provide a more visual demonstration of these distinct forms of structure in \cref{fig:strings}d, which shows ``dot plots'' for the highest-reward string using each reward function, as well as for a random string. Dot plots are commonly used in DNA analysis to identify repeated structures within a string \citep{gibbs1970diagram}, and build a symmetric binary $|x| \times |x|$ matrix where a given cell at index $(i, j)$ is $1$ if $x[i] = x[j]$ and $0$ otherwise. A repeated pattern appears in the plot as diagonal lines off of the main diagonal, making it easy to visually inspect the number, length, and locations of these patterns. Furthermore, if a pattern contains a smaller sub-pattern that appears elsewhere in the string, it will show up as a smaller diagonal line in the same block of rows spanned by the larger pattern (or, equivalently, block of columns due to the matrix's symmetry). For compositionality-optimized strings, we find that the dot plots contain a more dense distribution of patterns and sub-patterns than those of structure-optimized or random strings. Structure-optimized strings do still exhibit repeated patterns, but they tend to be longer, more sparse, and less hierarchical. Overall, the dot plots of structure-optimized strings look more similar to those of random strings, but the addition of a few long diagonals.

\subsection{Compositional images}

We next consider a more complex domain in which humans can readily recognize compositional structure: images. We generate $3$-channel images of size $64 \times 64$, with values discretized to $4$ possible intensities, resulting in a domain $\mathcal{X} = \{0,\dots,3\}^{64 \times 64 \times 3}$. Since nothing in our reward prescribes these images to be ``natural'', we expect them only to contain abstract compositional structure. Methods are briefly described below; \cref{sec:experiment_details_images} contains further details.

\begin{figure}[ht]
    \centering
    \includegraphics[width=\linewidth]{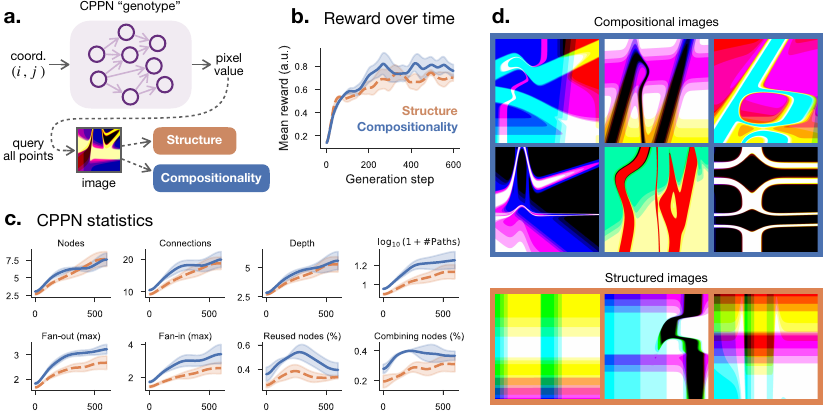}
    \caption{\textbf{Generating compositional images.}
    \textbf{a.} We evolve a population of images using either (i) a reward that maximizes multi-scale compositional structure using our definition (\textcolor{MidnightBlue}{\textbf{blue}}; \cref{eq:experiment_reward}) or (ii) an alternative reward that just maximizes structure without any pressure for compositionality (\textcolor{BurntOrange!50!RawSienna}{\textbf{orange}}; $K(m_x)$). Evolved genomes are Compositional Pattern-Producing Networks (CPPNs) that generate image phenotypes.
    \textbf{b.} Growth of the average population reward over time. Both the compositionality and structure rewards are successfully optimized. We plot in arbitrary units (a.u.) on the y-axis normalized based on the top score for each reward.
    \textbf{c.} Statistics of image-generating CPPNs over time, when optimizing each reward. Both rewards produce CPPNs of similar size (nodes and connections), but compositionality-optimized CPPNs exhibit greater feature recombination. See main text and \cref{sec:experiment_details_images} for descriptions of each statistic.
    \textbf{d.} Example images generated as a result of optimizing each reward. Optimizing for compositionality results in images with unified features, whereas optimizing for structure results in more disjoint images. More examples are provided in \cref{sec:experiment_details_images}.
    \textbf{Seeds.} Shaded regions in panels (b, c) reflect the standard error of the mean across $3$ seeds.}
    \label{fig:images}
\end{figure}

\paragraph{MDL algorithm.}

Fixed-pattern grammars are too restrictive to capture the shapes, textures, and spatial relationships that make images interesting. We therefore use a coordinate-based neural representation, following the general approach used by neural radiance fields and compositional pattern-producing networks (CPPNs) \citep{mildenhall2020nerf,stanley2007compositional}. For every image patch $x_v$, we fit a fixed-architecture network that maps a normalized two-dimensional coordinate to a categorical distribution over the four possible intensities in each channel. Because one set of weights must account for every coordinate, spatial regularities can be represented by the network while idiosyncratic pixels remain in the model-to-data code.

Merely training a flexible network to reproduce a single patch can lead it to memorize idiosyncratic pixels rather than capture structure that genuinely compresses the image. We therefore regularize the fitting procedure, most importantly by training with maximum likelihood on a random half of the patch's pixels and using the remaining half as held-out data for early stopping. The selected network is the checkpoint with the lowest held-out cross-entropy, and we retain it only if it improves upon a constant model that predicts the marginal intensity distribution of each channel. This fallback prevents the network from attributing spurious spatial structure to a patch that is better described without spatial dependence.

To complete the two-part code, we must estimate both the information stored in the fitted network and the information left unexplained by it. Following learning-curve estimates of neural-network description length \citep{blier2018description}, we estimate the first quantity from the cumulative training loss incurred while learning the network. At each point in training, we subtract the loss of the final selected model: the remainder measures predictive errors caused by not yet having learned the model, rather than errors that remain even after it has been learned. Accumulating this excess over training therefore provides a proxy for the number of bits absorbed into the network's weights. The model-to-data term is the selected model's negative log-likelihood over the complete patch. Together, these terms form our empirical two-part coding scheme $f$.

\paragraph{Conditional model compression.}

Directly describing one neural network given several others is complicated by the many parameter settings that implement the same function. Instead, we use the intuition that the conditional complexity of a target model is reflected by how much must be learned to fine-tune the given models into it. We implement this comparison at the level of predictions rather than parameters. At a join, the model $m_{x_v}^f$ fitted to the parent patch acts as a teacher. The four child models are combined into a routed student: each coordinate is sent to the model fitted to its quadrant, using that child's local coordinates. This unmodified student represents the hypothesis that the parent model is merely the juxtaposition of the four child models.

We then fine-tune copies of the child models to match the teacher's predictive distributions across the parent patch, again using held-out coordinates for early stopping. We retain the adapted student only when it improves upon the unmodified routed student on these coordinates. Our conditional code $g$ has two terms: a learning-curve estimate of the information required for this adaptation, computed by accumulating the distillation loss attributable to the student not yet having reached its selected state, and the remaining divergence between the selected student and teacher across the complete patch. If the structure of the whole consists only of the independent structures of its quadrants, with no novel structure revealed by considering them jointly, routing among the child models already reproduces the parent and both costs are near zero. Conversely, structure that becomes useful only when the quadrants are modeled together changes the parent's predictions in a way that requires substantial adaptation, producing a large conditional code. This gives a function-level surrogate for $L_{\mathrm{cond}}^g(m_{x_v}^f\mid m_{x_{v_1}}^f,\dots,m_{x_{v_4}}^f)$ without requiring the independently fitted networks to have aligned parameters.

\paragraph{Reward.}

Whereas a string has one spatial axis and can naturally be bisected, an image has two, so we recursively divide both axes to obtain a quadtree decomposition. Each patch is split into four quadrants until reaching $8\times8$ patches. The resulting joins combine $8\times8$ patches into $16\times16$ patches, then $16\times16$ patches into $32\times32$ patches, and finally the four $32\times32$ patches into the complete image. As in the string experiment, we normalize each conditional code by the number of spatial locations in its parent patch, average within each join scale, and combine the three scale averages with a power mean using $p=0.75$. This gives equal overall weight to each spatial scale while placing greater pressure on the weakest ones.

\paragraph{Optimization.}

The image space is far too large to search directly, so we search over compact procedures that generate images. Specifically, we use compositional pattern-producing networks (CPPNs): coordinate-based neural networks that specify a single image by mapping each normalized spatial coordinate to its channel values \citep{stanley2007compositional}. We can thus treat a CPPN's topology and parameters as a genotype and the image obtained by evaluating it across the coordinate grid as its phenotype. Its three continuous outputs are quantized to the four intensity levels in each channel.

We optimize a population of $200$ such genomes for $600$ generations using the NeuroEvolution of Augmenting Topologies (NEAT) \citep{stanley2002evolving,stanley2004competitive} algorithm. NEAT applies crossover and mutations to network weights, activation functions, nodes, and connections while maintaining species of related networks, allowing the complexity of the image-generating process to grow during search. This resembles the Picbreeder study \citep{secretan2008picbreeder,secretan2011picbreeder} (and we use very similar hyperparameters for the evolutionary algorithm), except that our compositionality reward replaces human selection. Both the generator and the empirical coding scheme use coordinate networks; the experiment therefore searches for spatial structure that the chosen coding scheme is capable of recognizing, rather than claiming to cover every possible form of image structure.

\paragraph{Results.}

We ran the evolutionary optimization for $600$ generations and only $3$ seeds, given the higher computational cost of this experiment. As in our string generation experiments, we compared compositionality-optimized images to structure-optimized ones. \cref{fig:images}b shows that both rewards increase over the course of generations before stabilizing.

We quantitatively analyze the structure of the resulting images by computing simple network summary statistics of the CPPN genomes that produce them in \cref{fig:images}c. We find that the average number of nodes in the evolved networks, as well as the connections between them, is comparable for the two rewards throughout the evolutionary process. The average network depth --- which measures the longest path (in connections) from an input to an output node --- also evolves similarly. The CPPNs instead seem to differ along more fine-grained metrics that suggest the compositionality reward produces more interconnected networks. For instance, the number of distinct input-to-output paths is greater when optimizing the compositionality reward. Features represented by hidden nodes are also reused by more than one downstream node more often, and fan-out to a greater number of nodes. The same is true for the inputs to these hidden units: multiple inputs are in general combined to form downstream features more often, and in greater number. Taken together, these results suggest that our compositionality reward results in unified structures and feature recombinations rather than independent structures, in line with our expectations.

Our generated images themselves produce further evidence along these lines; \cref{fig:images}d shows some representative examples, and more are provided in \cref{sec:experiment_details_images}. Compositionality-optimized images reliably look more unified: they have distinct central themes involving repeated motifs, shapes, and colour palettes (although it is difficult to visually establish the presence of feature hierarchies in an objective way). In contrast, structure-optimized images appear significantly more disjoint: a common trend is that they simply pack many bands of distinct horizontal and vertical colours into the image. While these bands provide learnable information, they are structurally independent from one another. These differences are precisely what our definition of compositional structure is intended to capture.

\section{Discussion}
\label{sec:discussion}

We defined compositionality as the novel shared structure revealed when parts are modeled together. Formally, it is the information in the best model of a whole that cannot be recovered from the best models of its parts. Applying this comparison at successive nodes of a decomposition tree yields a profile that reflects the locations and scales at which novel compositional structure emerges. We tested this definition by using it as a generative objective. Optimizing strings and images for prescribed profiles produced objects with recognizable nested reuse and coherent organization across scales. Because the objective constrains the form of this organization while remaining agnostic to its particular content, these results suggest that the definition isolates the intended part--whole organization rather than merely rewarding complexity, regularity, or shared information.

\paragraph{The success of deep learning.}

This view provides one explanation for why depth is a useful inductive bias. If natural data reveals novel structure whenever smaller regions are combined into larger ones, then an effective model should likewise build representations of larger-scale structure from those available at smaller scales. Successive layers of a deep network provide exactly this form of capacity: later features can represent combinations and reorganizations of earlier ones, without requiring the same composition rule at every scale. Standard architectures nevertheless impose a fixed maximum depth and are usually trained on \textit{iid} datapoints of a fixed extent. For nonstationary streams whose structure continues to grow across temporal scales, the framework instead motivates more flexible architectures whose effective depth or hierarchy can expand whenever larger contexts reveal additional compositional structure. Reasoning methods \citep{wei2022chain} and looped layers \citep{dehghani2018universal} provide potential building blocks towards this goal.

\paragraph{Curricula and open-ended learning.}

The same idea shifts emphasis from specifying a compositional model directly to constructing data whose compression requires compositional structure. Let $x_a$ be a learner's history and $x_b$ a new segment of experience. A large $C(x_a,x_b\to x)$ means that $x_b$ does more than repeat known structure, but also does more than contribute an unrelated fact: considered with the history, it induces a new joint account that reuses or reorganizes earlier structure. A curriculum can therefore be scored by the local compositionality introduced at each step, or by a functional of the resulting left-to-right compositional structure profile. Searching over orderings could then test whether curricula with high scores produce representations that are more reusable and factorized. This interpretation is consistent with the finding of \citet{kumar2025questioning} that networks evolved through human-guided, open-ended search develop more unified and factorized representations than networks trained only to reproduce their endpoints. Our framework gives a precise hypothesis for this difference: the evolutionary history contributes substantial compositional structure that is absent from the endpoint alone.

This objective can also guide an agent that chooses its own experience. A policy could be rewarded for selecting $x_b$ that maximizes an empirical approximation to $C(x_a,x_b\to x)$, thereby constructing a curriculum through interaction. This is closely related to intrinsic curiosity based on compression progress \citep{schmidhuber2009driven,schmidhuber2010formal} and to objectives that seek learnable structure such as Epiplexity \citep{finzi2026epiplexity}. The difference is the constraint placed on \emph{how} the learner's model changes: the reward favours novel structure that becomes useful through its relationship to prior experience, rather than novelty or learnability alone. Repeated over time, this provides an operational criterion for open-ended learning in which knowledge continues to grow by building on itself.

\paragraph{Future work and limitations.}

The ideal quantities in our definition are uncomputable, and the empirical version depends on restricted model classes, imperfect MDL search, and conditional codes that provide only operational surrogates for algorithmic complexity. Our experiments show that useful approximations are possible, but not how stable the resulting scores are across compressors or optimization procedures. Establishing this robustness, and identifying which approximation errors preserve comparisons between objects, are important next steps.

A second, crucial open question that we highlighted in \cref{sec:defining} is: which decompositions should be considered for a particular object? The compositional structure profiles we defined here are relative to a decomposition tree, and different trees can expose different structures in the same object. However, is there any sense in which an object admits an ``intrinsic'' decomposition, and if so how can it be defined? A systematic account should determine which decompositions expose useful, robust structure and which merely introduce artifacts. We conjectured that the tree which maximizes the sum of compositionality scores is a particularly good candidate for a canonical decomposition, but this must be further investigated. Ultimately, these questions reflect a broader feature of compositionality: any claim about how a whole is composed must identify its constituent parts. Traditional accounts often fix these constituents in advance through a syntax, a set of semantic factors, or an object decomposition; our framework makes the choice explicit, allowing the decomposition itself to become an object of study.

\bibliographystyle{apalike}
\bibliography{references}

\begin{appendices}
\crefalias{section}{appendix}   
\crefalias{subsection}{appendix} 
\counterwithin{figure}{section} 
\counterwithin{table}{section}  

\section{Practical approximation of compositional structure}
\label{sec:approximating_compositional_structure_details}

The definitions in the main text are uncomputable, since they are rooted in algorithmic information theory. In practice, however, they can be approximated using empirical compressors, model classes, learning algorithms, and minimum description length procedures.

Estimating compositionality empirically involves three distinct difficulties:
\begin{enumerate}
    \item $K(m)$ is uncomputable, yet is required to define the Occam's razor model $m_x$ in \cref{eq:occam}.
    \item $K(m_x\mid m_{x_1},\dots,m_{x_k})$, which defines the amount of novel compositional structure in \cref{def:compositionality}, is likewise uncomputable.
    \item Identifying $m_x$ requires searching over the potentially infinite model class $\mathcal{M}$ and determining that no unexamined model yields a shorter two-part description.
\end{enumerate}

The third difficulty is distinct from evaluating the complexity of any particular candidate model. Kolmogorov complexity is upper semicomputable: by running candidate programs in parallel and retaining the shortest description found so far, one obtains successively tighter upper bounds that eventually reach the optimal value, although there is no general way to recognize when this has occurred. \citet{bedard2022algorithmic} make the analogous observation for Kolmogorov's structure function. This gives the theoretical search an \emph{anytime} interpretation: continued search can reveal better models and shorter descriptions, but it does not provide a computable certificate that the current model is optimal. A practical MDL procedure replaces this open-ended search with a specified model class, optimization method, and finite search budget.

We replace the uncomputable quantities with empirical code lengths induced by practical coding schemes. Let $f$ denote an empirical two-part coding scheme used to fit models by minimum description length, together with its associated model representation and search procedure. Let $g$ denote an empirical conditional coding scheme used to measure the additional cost of describing the model of the whole given the models of the parts. These coding schemes may be specialized to the particular data modality, model class, and training procedure under consideration.

Using the two-part coding scheme $f$, the Occam's razor model $m_x$ is replaced by an empirical MDL model. Let:
\begin{align}
    L_m^f(x)
    :=
    L_{\mathrm{model}}^f(m)
    +
    L_{\mathrm{data}}^f(x\mid m) ,
\end{align}
where $L_{\mathrm{model}}^f(m)$ is the code length of the model $m$, and $L_{\mathrm{data}}^f(x\mid m)$ is the code length of the data $x$ given that model. The ideal empirical minimum-description-length models are:
\begin{align}
    \mathcal{M}^{f,*}(x)
    :=
    \argmin_{m\in\mathcal{M}} L_m^f(x) .
\end{align}
Assuming there is a unique model in $\mathcal{M}^{f,*}(x)$ with minimal model-code length, define the ideal empirical Occam's razor model by:
\begin{align}
    m_x^{f,*}
    :=
    \argmin_{m\in\mathcal{M}^{f,*}(x)}
    L_{\mathrm{model}}^f(m) .
\end{align}

Thus, $m_x^{f,*}$ first minimizes the total two-part code length under $f$ and then, among ties, minimizes the model-description part of the code. This is the empirical analogue of the theoretical construction in \cref{eq:occam}. In practice, a finite run of the search procedure generally returns a best-so-far model $m_x^f$, without guaranteeing that $m_x^f=m_x^{f,*}$. If the procedure retains its incumbent, increasing the search budget cannot worsen the best observed two-part code length, although it may change both the selected model and the allocation of bits between its model and residual descriptions.

The tie-breaking distinction remains conceptually important even when exact ties in total code length are rare. Depending on the coding scheme, information may be shifted between the model-description term and the data-given-model term, so different models can represent different allocations of complexity even when their total description lengths are similar or identical. Rather than explicitly solving the two-stage optimization above, a practical implementation may use a single two-part MDL algorithm that simultaneously accounts for model-description length, residual code length, and the search for a suitable model. The choice of optimizer and search budget is therefore part of the empirical procedure $f$. For brevity, the main text writes $m_x^f$ for the model returned by this procedure.

Given the resulting empirical models $m_x^f,m_{x_1}^f,\dots,m_{x_k}^f$, the corresponding empirical estimate of compositionality under the conditional coding scheme $g$ is:
\begin{align}
    C^{f,g}(x_1,\dots,x_k\to x)
    :=
    L_{\mathrm{cond}}^g
    \bigl(
        m_x^f
        \mid
        m_{x_1}^f,\dots,m_{x_k}^f
    \bigr) ,
\end{align}
where $L_{\mathrm{cond}}^g(\cdot\mid\cdot)$ denotes the conditional code length induced by $g$. Thus, $C^{f,g}(x_1,\dots,x_k\to x)$ should be interpreted as a coding-strategy-dependent estimate of compositional structure: it measures how many additional bits are required, under $g$, to describe the empirically fitted model of the whole given the empirically fitted models of the parts.

This is not an estimate in the sense of providing formal upper or lower bounds on the underlying algorithmic-information quantities. Rather, it yields an operational notion of compositionality relative to a chosen pair of empirical procedures: a procedure $f$ for fitting two-part MDL models through a finite search, and a procedure $g$ for conditionally compressing the whole-model given the part-models.

\section{The maximal decomposition}
\label{sec:maximal_decomposition}

In \cref{def:compositionality_profile}, we introduced the concept of a compositional structure profile $\mathcal{P}_T(x) := \left(C_{T,x}(v)\right)_{v\in \mathrm{Int}(T)}$
, which is defined with respect to a rooted decomposition tree $T$. We then asked whether or not a given object $x$ admits an ``intrinsic'', ``true'', or ``canonical'' decomposition that reflects its genuine compositional structure. We conjectured that the decomposition which maximizes the sum of novel compositional structure quantities throughout the tree was a promising candidate. Here, we provide some initial arguments that motivate this conjecture.

Denote the sum of novel compositional structure quantities under a given decomposition $T$ as:
\begin{align}
    S_T(x) := \sum_{v\in\mathrm{Int}(T)} C_{T,x}(v) ,
\end{align}
and define the maximal decomposition of an object $x$ as:
\begin{align}
    T_x^{\mathrm{max}} := \argmax_{T} S_T(x) .
\end{align}

Intuitively, what the maximal tree accomplishes is that it preferentially joins objects that share novel structure and parts, which maps on to how we intuitively think about compositionality. In contrast, if we were considering \emph{minimal} decompositions, we would be joining objects that are either structurally identical or algorithmically independent. At a high level, this is the strongest motivation for our conjecture.

Second, while $S_{T_x^{\mathrm{max}}}(x)$ is lower-bounded by $K(m_x)$ (it provides \emph{a} possible strategy for compressing the final model, but not necessarily the optimal one; see \cref{sec:proofs_stx_kmx} for theorem and proof), there are strings for which $S_{T_x^{\mathrm{max}}}(x) > K(m_x)$. This means that $S_{T_x^{\mathrm{max}}}(x)$ does not simply reduce to sophistication but rather quantifies something different. It remains to be shown, theoretically and empirically, that $S_{T_x^{\mathrm{max}}}(x) > K(m_x)$ for objects we intuitively consider compositional, and conversely that $S_{T_x^{\mathrm{max}}}(x) \approx K(m_x)$ for objects we believe to be uncompositional. In contrast, it is trivial to show that if we were to instead consider \emph{minimal} decompositions $T_x^{\mathrm{min}}$, there would always exist a trivial solution that merges every individual bit of a string in a single join and results in $S_{T_x^{\mathrm{min}}}(x) = K(m_x) ,\;\forall x$ (see \cref{sec:proofs_stx_kmx} for theorem and proof) --- in other words, the minimal decomposition would be independent of $x$ and would thus never capture its intrinsic compositional structure.

Third, we should hope that if a string is measured to be ``compositional'' according to some decomposition, it remains compositional under that same decomposition when it is embedded as a \emph{substring} within a larger structure. If this were not the case, it would mean that the compositional structure of an object is context-dependent, and we would not be able to say that compositionality is an intrinsic and objective property of the object. $T_x^{\mathrm{max}}$ happens to have this property, so we can say that it is a recursively-consistent metric. In particular, whenever a substring occurs as a node of a maximizing tree, the subtree rooted at that node is itself a maximizing decomposition of the substring (see \cref{sec:proofs_recursive_consistency} for theorem and proof)

Fourth, considering the maximal tree solves a technical problem: why consider binary joining operations as we have primarily focused on here? Why not join triplets, or quadruplets of substrings? While this choice may seem arbitrary, it is cleanly resolved if we only consider maximal decompositions. The reason is that the novel structure generated from joining a triplet is guaranteed to be less than or equal to what is generated from a pair of binary joining operations, where we first join two of the substrings, and then join the result with the remaining substring (see \cref{sec:proofs_binary_refinement} for theorem and proof). Once again, minimal decompositions give the opposite result because they always admit a trivial solution that merges every individual bit of the string in a single join.

\section{Proofs}
\label{sec:proofs}

\subsection{Proof of the relationship between compositionality and algorithmic mutual information}
\label{sec:proofs_ami}

\begin{theorem}[Compositionality implies algorithmic mutual information]
\label{thm:ami}
Adopt the notation from  \cref{sec:defining}. Suppose that a conditionally novel model of the whole does not tie the separate two-part descriptions of the parts; that is:
\begin{align}
    \label{eq:no_neutral_ties}
    C(x_a,x_b\to x)>0
    \quad\implies\quad
    K(m_x)+l_{m_x}(x)
    \neq
    K(m_{x_a})+l_{m_{x_a}}(x_a)
    +
    K(m_{x_b})+l_{m_{x_b}}(x_b) .
\end{align}
Then:
\begin{align}
    C(x_a,x_b\to x)>0
    \implies
    I(x_a:x_b)>0 ,
\end{align}
but reverse implication does not hold in general.
\end{theorem}

Intuitively, the theorem states that novel structure appearing only when the parts are modeled jointly generally makes the whole more compressible than the parts considered separately, and therefore implies positive algorithmic mutual information. The no-tie assumption excludes the exceptional case in which changing to a novel model of the whole reduces the model-description cost but increases the residual code length by exactly the same amount, leaving the total two-part description length unchanged. This is a mild assumption: changing the model will ordinarily change the total two-part compression cost, whereas an exact cancellation between its model and residual components requires a precise tie. The formal proof of the theorem follows.

\begin{notebox}{Technical status of the mutual information result}
    The above theorem and proof below use the paper-wide convention of suppressing additive coding terms and are therefore an idealized description-length argument, not a literal theorem about prefix Kolmogorov complexity without further qualifications. With full bookkeeping, combining independent descriptions, identifying an optimal two-part code with $K(x)$, and writing $K(m_x \mid m_{x_a}, m_{x_b}) = 0$ each introduce machine-dependent overheads; the relevant bounds may also contain logarithmic terms when lengths or boundaries must be encoded. Moreover, because $x=[x_a,x_b]$ denotes raw concatenation, identifying $K(x)$ with the joint complexity $K(x_a,x_b)$ assumes that the decomposition boundary is available as side information. Consequently, the strict tests $C>0$ and $I>0$ can be altered by the suppressed terms. A fully rigorous variant would replace zero by significance thresholds exceeding the relevant overheads and strengthen the no-tie assumption to require a compression margin exceeding them. The idealized result is retained because it isolates the intended qualitative relationship: novel joint structure ordinarily yields a compression advantage, whereas mutual information can also arise merely from reusing structure already present in the parts.
\end{notebox}

\begin{proof}
The two parts can always be compressed independently by concatenating their model descriptions and residual codes. Because \(m_x\) is an optimal model of the whole, we have that:
\begin{align}
    K(m_x)+l_{m_x}(x)
    \leq
    K(m_{x_a})+l_{m_{x_a}}(x_a)
    +
    K(m_{x_b})+l_{m_{x_b}}(x_b) .
\end{align}
If \(C(x_a,x_b\to x)>0\), the no-tie assumption in \cref{eq:no_neutral_ties} makes this inequality strict. Therefore:
\begin{align}
    K(x)
    &=
    K(m_x)+l_{m_x}(x) \\
    &<
    K(m_{x_a})+l_{m_{x_a}}(x_a)
    +
    K(m_{x_b})+l_{m_{x_b}}(x_b) \\
    &=
    K(x_a)+K(x_b) .
\end{align}
Hence:
\begin{align}
    I(x_a:x_b)
    =
    K(x_a)+K(x_b)-K(x)
    >
    0 .
\end{align}

For the reverse direction, consider a case in which the same model \(m\) is the Occam's razor model of \(x_a\), \(x_b\), and \(x\), and write \(c=K(m)>0\). The whole can reuse this model description:
\begin{align}
    K(x)
    &\leq
    c+l_m(x_a)+l_m(x_b) \\
    &<
    2c+l_m(x_a)+l_m(x_b) \\
    &=
    K(x_a)+K(x_b) .
\end{align}
Thus \(I(x_a:x_b)>0\). Nevertheless, we have that:
\begin{align}
    C(x_a,x_b\to x)
    =
    K(m\mid m,m)
    =
    0 .
\end{align}
Therefore, positive algorithmic mutual information does not imply novel compositional structure.
\end{proof}

\subsection{Proofs of the relationships between \texorpdfstring{$S_T(x)$ and $K(m_x)$}{ST(x) and K(mx)}}
\label{sec:proofs_stx_kmx}

This section establishes the bounds used in \cref{sec:maximal_decomposition}. \Cref{thm:total_compositionality} relates $S_T(x)$ to $K(m_x)$ for an arbitrary decomposition, accounting explicitly for the tree shape and leaf models. \Cref{cor:fully_refined_maximum,cor:maximal_compositionality_bound} show that a maximizing decomposition can be chosen with single-bit leaves and, assuming a common empty model for individual bits, establish the corresponding lower bound on maximal compositionality. Finally, \cref{cor:minimal_complete} describes the trivial single-join solution for complete minimal decompositions, while \cref{cor:minimal_unrestricted} shows that allowing incomplete decompositions yields a zero minimum.

\begin{theorem}[Model complexity bounds total compositionality]
\label{thm:total_compositionality}
Let $T$ be a finite rooted, ordered decomposition tree for $x$,
and suppose that the Occam's razor model $m_{x_v}$ is
well-defined at every node $v$. Let $\tau_T$ denote the
unlabeled ordered tree shape, and let
$M_{\mathrm{leaf}} := (m_{x_\ell})_{\ell\in\mathrm{Leaves}(T)}$
list the leaf models in left-to-right order. Then:
\begin{align}
    S_T(x)
    \geq K(m_x) - K(\tau_T,M_{\mathrm{leaf}}) - O(1),
\end{align}
where the constant is independent of $x$ and $T$.
\end{theorem}

\begin{proof}
For each internal node $v$ with ordered children
$v_1,\ldots,v_{k_v}$, choose a shortest prefix program $p_v$
that outputs $m_{x_v}$ given
$(m_{x_{v_1}},\ldots,m_{x_{v_{k_v}}})$. By definition,
\begin{align}
    |p_v|
    = K(m_{x_v}\mid m_{x_{v_1}},\ldots,m_{x_{v_{k_v}}})
    = C_{T,x}(v).
\end{align}
Construct a description consisting of a shortest prefix
program for $(\tau_T,M_{\mathrm{leaf}})$, followed by the
programs $p_v$ in postorder, so that children precede
their parents.

A fixed decoder first recovers the tree shape and leaf
models. It then processes the internal nodes in postorder,
running each $p_v$ with the already reconstructed child
models as its auxiliary input. Since the programs are
prefix-free and the tree determines their number and
order, no additional delimiters are required. At the
root, the decoder outputs $m_x$. Consequently,
\begin{align}
    K(m_x)
    &\leq K(\tau_T,M_{\mathrm{leaf}})
       + \sum_{v\in\mathrm{Int}(T)} |p_v| + O(1) \\
    &= K(\tau_T,M_{\mathrm{leaf}}) + S_T(x) + O(1).
\end{align}
Rearranging proves the result.
\end{proof}

\begin{corollary}[Maximal decompositions can be fully refined]
\label{cor:fully_refined_maximum}
Let $x\in\{0,1\}^n$, with $n\geq 1$. Consider all rooted,
ordered decomposition trees of $x$ whose internal nodes
have at least two children and whose nodes represent
nonempty substrings. There exists a tree maximizing
$S_T(x)$ whose leaves are individual bits.
\end{corollary}

\begin{proof}
Let $\ell$ be a leaf of a decomposition tree $T$, and
let $R$ be a decomposition tree of $x_\ell$. Replacing
$\ell$ by $R$ produces a refinement $T'$ satisfying
\begin{align}
    S_{T'}(x) = S_T(x) + S_R(x_\ell) \geq S_T(x).
\end{align}
Indeed, every existing internal node retains the same
substring and immediate child substrings, so its score
is unchanged, while all newly introduced scores are
nonnegative.

There are finitely many admissible decomposition trees,
so a maximizer exists. Starting from any maximizer,
recursively split each leaf of length greater than one.
The resulting tree has individual bits as leaves and
cannot have a smaller score. Since the original tree
was maximal, the refined tree is also maximal.
\end{proof}

\begin{corollary}[Maximal compositionality bounds model complexity]
\label{cor:maximal_compositionality_bound}
Suppose that the Occam's razor models of both one-bit
strings coincide with a fixed empty model:
\begin{align}
    m_0 = m_1 = m_{\varnothing}.
\end{align}
For every nonempty binary string $x$ of length $n$,
any maximizing decomposition $T_x^{\mathrm{max}}$ satisfies
\begin{align}
    S_{T_x^{\mathrm{max}}}(x)
    \geq K(m_x) - K(n) - O(1).
\end{align}
Consequently, suppressing an additive
$O(\log(n+2))$ term gives
$S_{T_x^{\mathrm{max}}}(x)\geq K(m_x)$.
\end{corollary}

\begin{proof}
Let $B_n$ be the binary decomposition obtained by
recursively splitting each substring into two parts
whose lengths differ by at most one, stopping at
individual bits. Its ordered tree shape is computable
from $n$, and its leaf-model tuple consists of $n$
copies of the fixed model $m_{\varnothing}$. Hence
\begin{align}
    K(\tau_{B_n},M_{\mathrm{leaf}})
    \leq K(n)+O(1).
\end{align}
Applying \cref{thm:total_compositionality} gives
\begin{align}
    S_{B_n}(x) \geq K(m_x)-K(n)-O(1).
\end{align}
Finally, maximality implies
$S_{T_x^{\mathrm{max}}}(x)\geq S_{B_n}(x)$,
which proves the claim.
\end{proof}

\begin{corollary}[Minimal complete decompositions]
\label{cor:minimal_complete}
Let $x\in\{0,1\}^n$, with $n\geq 2$, and suppose
$m_0=m_1=m_{\varnothing}$ for a fixed empty model.
Let $\mathcal{T}_{\mathrm{complete}}(x)$ consist of all
rooted, ordered decomposition trees with single-bit
leaves and at least two children per internal node.

The single-join tree $T_\star$, whose root has all $n$
bits as its immediate children, satisfies
\begin{align}
    S_{T_\star}(x)=K(m_x\mid n)+O(1).
\end{align}
Moreover,
\begin{align}
    K(m_x)-O(n)
    \leq \min_{T\in\mathcal{T}_{\mathrm{complete}}(x)}
             S_T(x)
    \leq K(m_x\mid n)+O(1).
\end{align}
Under the convention of suppressing leaf-description
and accumulated coding terms, these bounds give
\begin{align}
    \min_{T\in\mathcal{T}_{\mathrm{complete}}(x)}
        S_T(x)=K(m_x),
\end{align}
with the single-join tree providing a trivial minimizer
in this idealized accounting.
\end{corollary}

\begin{proof}
The only internal node of $T_\star$ is its root, so
\begin{align}
    S_{T_\star}(x)
    =K\bigl(m_x\mid
      \underbrace{m_{\varnothing},\ldots,m_{\varnothing}}
                  _{n\text{ copies}}\bigr)
    =K(m_x\mid n)+O(1).
\end{align}
The final equality holds because the tuple determines
$n$, and conversely $n$ determines the tuple using the
fixed model $m_{\varnothing}$. This proves the upper bound.

For the lower bound, every complete tree has $n$ leaves
and at most $n-1$ internal nodes. Its ordered shape can
therefore be encoded in $O(n)$ bits. Since all leaf
models equal $m_{\varnothing}$, the shape also determines
their tuple, giving
\begin{align}
    K(\tau_T,M_{\mathrm{leaf}})=O(n).
\end{align}
By \cref{thm:total_compositionality}, uniformly over
complete trees,
\begin{align}
    S_T(x)\geq K(m_x)-O(n).
\end{align}
Taking the minimum proves the stated bounds. Suppressing
the indicated coding terms yields the idealized equality.
\end{proof}

\begin{corollary}[Minimal unrestricted decompositions]
\label{cor:minimal_unrestricted}
If decomposition trees may terminate at arbitrary
nonempty substrings, then for every nonempty string $x$,
\begin{align}
    \min_T S_T(x)=0.
\end{align}
A minimizer is the tree consisting only of the root,
labelled by $x$.
\end{corollary}

\begin{proof}
The root-only tree has no internal nodes, so its score
is an empty sum and equals zero. Every decomposition
score is a sum of nonnegative conditional complexities
and is therefore nonnegative.
\end{proof}

\subsection{Recursive consistency of maximal decompositions}
\label{sec:proofs_recursive_consistency}

\begin{theorem}[Recursive consistency of maximal decompositions]
\label{thm:recursive_consistency}
Fix the universal machine, the tuple encoding, and a uniquely selected Occam's razor model $m_y$ for every nonempty substring $y$ of $x$. Let $\mathcal{T}(y)$ be the set of finite rooted, ordered decomposition trees of $y$ whose nodes represent nonempty substrings and whose internal nodes have at least two children. Define
\begin{align}
    S^{\mathrm{max}}(y)
    := \max_{R\in\mathcal{T}(y)} S_R(y).
\end{align}
If $T\in\mathcal{T}(x)$ maximizes $S_T(x)$, then, for every node $v$ of $T$, the subtree $T_v$ rooted at $v$ satisfies
\begin{align}
    S_{T_v}(x_v) = S^{\mathrm{max}}(x_v).
\end{align}
Moreover, replacing $T_v$ by any maximizing decomposition of $x_v$ produces another maximizing decomposition of $x$.
\end{theorem}

\begin{proof}
Each set $\mathcal{T}(y)$ is finite and nonempty, so its maximum exists. Fix a node $v$ of $T$ and any $R\in\mathcal{T}(x_v)$. Replace $T_v$ by $R$, identifying the root of $R$ with $v$, and denote the resulting tree by $T'$. This is an admissible decomposition of $x$.

Every internal node outside the replaced subtree retains the same substring and the same ordered list of immediate child substrings. In particular, the parent of $v$, if present, still has $x_v$ as the same child substring. Since each model depends only on its substring, all local scores outside the replaced subtree remain unchanged. Consequently,
\begin{align}
    S_{T'}(x)
    = S_T(x) - S_{T_v}(x_v) + S_R(x_v).
\end{align}
Maximality of $T$ implies $S_{T'}(x)\leq S_T(x)$, and hence $S_R(x_v)\leq S_{T_v}(x_v)$. Since $R$ was arbitrary, $T_v$ is a maximizing decomposition of $x_v$. If $R$ is also maximizing, the displayed identity gives $S_{T'}(x)=S_T(x)$, proving the final claim.
\end{proof}

\subsection{Binary refinement of maximal decompositions}
\label{sec:proofs_binary_refinement}

\begin{theorem}[Binary refinement of maximal decompositions]
\label{thm:binary_refinement}
Fix the universal prefix machine, the tuple encoding, and a uniquely selected Occam's razor model $m_y$ for every nonempty substring $y$ of a nonempty string $x$. Let $\mathcal{T}(x)$ consist of the finite rooted, ordered decomposition trees of $x$ whose nodes represent nonempty substrings and whose internal nodes have at least two children. Write
\begin{align}
    S_T(x) &:= \sum_{v\in\mathrm{Int}(T)} K(m_{x_v}\mid m_{x_{v_1}},\ldots,m_{x_{v_{k_v}}}), \\
    N_{>2}(T) &:= \bigl|\{v\in\mathrm{Int}(T):k_v>2\}\bigr|.
\end{align}
There is a constant $c\geq 0$, independent of $x$ and $T$, such that every $T\in\mathcal{T}(x)$ admits a binary refinement $B$ with the same ordered leaves and
\begin{align}
    S_B(x) \geq S_T(x)-cN_{>2}(T).
    \label{eq:binary_refinement_bound}
\end{align}
Consequently, under the convention of suppressing these accumulated additive coding terms, a maximizing decomposition can be chosen to consist entirely of binary joins.
\end{theorem}

\begin{proof}
First consider a single join $y=[y_1,\ldots,y_k]$, where $k\geq 3$, and set $z_j=[y_1,\ldots,y_j]$ for $1\leq j\leq k$. Its left-associated binary refinement joins $z_{j-1}$ with $y_j$ to form $z_j$, successively for $j=2,\ldots,k$. We claim that
\begin{align}
    K(m_y\mid m_{y_1},\ldots,m_{y_k})
    \leq \sum_{j=2}^{k} K(m_{z_j}\mid m_{z_{j-1}},m_{y_j})+c.
    \label{eq:binary_join_bound}
\end{align}
For each $j$, choose a shortest prefix program $p_j$ that outputs $m_{z_j}$ given $(m_{z_{j-1}},m_{y_j})$. Given the tuple $(m_{y_1},\ldots,m_{y_k})$, a fixed decoder reads the concatenated programs $p_2\cdots p_k$. It starts with $m_{z_1}=m_{y_1}$ and recursively runs $p_j$ with the previously recovered model $m_{z_{j-1}}$ and the supplied model $m_{y_j}$. The final output is $m_{z_k}=m_y$. The input tuple determines $k$, and prefix-freeness permits the programs to be parsed successively without additional delimiters. Thus their total length, plus a fixed simulation constant $c$, bounds the conditional complexity of $m_y$, proving \eqref{eq:binary_join_bound}. In particular, $c$ does not depend on $k$. This construction does not require computing an Occam's razor model from its substring: each intermediate model is supplied by its corresponding program.

Now recursively refine the original tree. At each node with more than two children, replace its join by the left-associated binary chain just described, preserving the original child subtrees. The node still represents the same substring, so its parent's score is unchanged; scores within the original child subtrees are also unchanged. \Eqref{eq:binary_join_bound} shows that the replacement decreases the total score by at most $c$. Nodes already having two children need no replacement. After replacing all $N_{>2}(T)$ nonbinary joins, the resulting tree $B$ is binary and has the same ordered leaves. Summing the local bounds proves \eqref{eq:binary_refinement_bound}.

There are finitely many admissible trees, so a maximizer exists. In the idealized accounting that suppresses the indicated coding terms, refining a maximizer cannot decrease its score, and maximality prevents an increase beyond the maximum. Its binary refinement is therefore also a maximizer in that accounting.
\end{proof}

\section{Experiment details}
\label{sec:experiment_details}

\subsection{Compositional strings}
\label{sec:experiment_details_strings}

\paragraph{Grammar representation and codelength.}

For a substring $x_v$, the two-part compressor returns a grammar $G=(A_i\to\alpha_i)_{i=1}^M$ and a residual sequence $z$. Each right-hand side $\alpha_i$ is a sequence of original symbols and rule names $A_j$ with $j<i$, so the rules form an acyclic hierarchy. Recursively expanding every rule name in $z$ recovers $x_v$ exactly. The empirical model $m_{x_v}^f$ is the grammar $G$ alone; $z$ is the data-given-model component and is consequently not passed to the conditional model compressor.

We use a shared zero-order entropy code for the grammar and residual. Let $\mathcal{S}(G,z)$ be the pooled stream containing every symbol on the right-hand side of every rule, one occurrence of each rule name on its left-hand side, and every symbol in $z$. If $n_s$ is the number of occurrences of symbol $s$ in this stream and $N=\sum_s n_s$, its codelength is $\lambda(s)=-\log_2(n_s/N)$. We separate the resulting description into:
\begin{align}
    L_{\mathrm{model}}^f(G)
    &= \sum_{i=1}^M\left[\lambda(A_i)+\sum_{s\in\alpha_i}\lambda(s)\right], \\
    L_{\mathrm{data}}^f(x_v\mid G)
    &= \sum_{s\in z}\lambda(s), \\
    L^f_G(x_v) &= L_{\mathrm{model}}^f(G)+L_{\mathrm{data}}^f(x_v\mid G).
\end{align}
Including one left-hand-side occurrence charges for introducing each rule, while occurrences of its name in other rules or in $z$ capture the savings from reuse. Because the empirical symbol probabilities are estimated from the pooled stream, adding a rule changes both the sequence length and all symbol codelengths; rule induction must therefore evaluate the change in the complete two-part code rather than count replaced symbols alone. We treat rule boundaries and the empirical distribution as part of the agreed coding convention and do not encode their additional overhead. Thus, as in our broader empirical framework, these quantities should be interpreted as consistently applied surrogate codelengths rather than calibrated lengths of a particular binary file format.

\paragraph{MDL grammar induction.}

Standard RePair maintains the current residual sequence $z$, finds its most frequent adjacent pair of symbols, introduces a new rule for that pair, and replaces the pair's non-overlapping occurrences by the new rule name \citep{larsson2000off}. It repeats this process until no pair occurs twice. For example, replacing $01$ by $A$ changes $01010$ to $AA0$ and adds the rule $A\to 01$. Since the symbols of $z$ can include previously introduced rule names, later replacements can build a hierarchy of rules. Two features of standard RePair are undesirable for our purpose. First, it discovers a long repeated substring through several pair replacements, even when one rule for the entire substring would be shorter and easier to interpret. MR-RePair addresses this problem by starting from a repeated pair and extending all of its occurrences to the left and right for as long as their surrounding symbols agree, then introducing a rule for the resulting longer repeat \citep{furuya2018mr}. Second, RePair and MR-RePair ordinarily add a rule whenever a suitable repetition exists. They do not ask whether the saved occurrences outweigh the cost of storing the rule, and can therefore assign substantial grammars even to random strings for which the empty grammar gives a shorter two-part code.

Our fitter retains MR-RePair's long-repeat search but chooses among its candidates by the resulting change in $L_G^f(x_v)$; \cref{alg:mdl_grammar_induction} gives the complete procedure. At each iteration, it considers every distinct adjacent pair with at least two non-overlapping occurrences in $z$. For one such pair, it collects all of the pair's occurrences and extends them simultaneously: it prepends a symbol only if the same symbol appears immediately before every occurrence, and appends a symbol only if the same symbol appears immediately after every occurrence. This continues until neither extension is possible. The result is maximal relative to this set of occurrences; it need not be the longest repeat anywhere in $z$. As in MR-RePair, if the result has length greater than two and its first and last symbols are equal, we remove its final symbol. This boundary adjustment prevents shifted, self-overlapping copies from inducing inconsistent replacements. We then greedily retain the leftmost possible non-overlapping occurrences. If at least two remain, the resulting pattern $q$ becomes a candidate right-hand side for a new rule $A\to q$. Its exact gain is
\begin{align}
    \Delta(q) = L_G^f(x_v)-L_{G\cup\{A\to q\}}^f(x_v),
\end{align}
where the second term is evaluated after replacing those occurrences by $A$. The selected candidate is the one with greatest gain, and a rule is added only if this gain is positive. Thus, every accepted step strictly shortens the surrogate two-part code.

More concretely, the pooled entropy length can be written as $N\log_2N-\sum_s n_s\log_2n_s$. If a candidate $q$ has length $m$, contains symbol $s$ exactly $c_s(q)$ times, and is replaced at $k$ non-overlapping occurrences, then introducing its new rule name $A$ gives:
\begin{align}
    N' &= N+(m+1)-k(m-1), \\
    n_s' &= n_s-(k-1)c_s(q), \qquad
    n_A'=k+1.
\end{align}
The $m$ right-hand-side symbols and one left-hand-side occurrence account for the $m+1$ added model symbols, while replacing $k$ copies of an $m$-symbol pattern by $k$ rule names accounts for the residual reduction. Substituting these counts into the entropy expression yields $\Delta(q)$ without rebuilding the candidate grammar.

The pair-extension procedure has an important blind spot. It extends a pair only while \emph{all} occurrences of that pair have the same neighbouring symbols. Suppose, for example, that several copies of a long pattern contain the pair $01$, but $01$ also occurs elsewhere in unrelated contexts. Those extra occurrences can stop the simultaneous extension immediately, even though the subset inside the long pattern would yield a valuable rule. We therefore supplement the local pair search with a global repeat search based on a suffix array. This second search can isolate a subset of starting positions that share a long pattern.

A \emph{substring} is any contiguous segment of a sequence. A \emph{prefix} is a substring that begins at the first position, while the suffix beginning at position $i$ is the substring $z_i,z_{i+1},\ldots$ that runs from $i$ to the end. The suffix array does not store these suffixes themselves: it stores their starting positions, ordered so that the corresponding suffixes are in lexicographic, or dictionary, order. For example, the suffixes of $z=\mathtt{cabdab}$, shown in sorted order, are
\begin{align}
    \mathtt{ab}_{(5)},\quad
    \mathtt{abdab}_{(2)},\quad
    \mathtt{b}_{(6)},\quad
    \mathtt{bdab}_{(3)},\quad
    \mathtt{cabdab}_{(1)},\quad
    \mathtt{dab}_{(4)},
\end{align}
where each subscript is the suffix's starting position in $z$. The suffix array is therefore $(5,2,6,3,1,4)$. Its first two suffixes begin with the same length-two prefix $\mathtt{ab}$, revealing occurrences of $\mathtt{ab}$ at positions $2$ and $5$. More generally, suffixes with a long shared prefix become neighbours in this ordering, even if other occurrences of the prefix's shorter symbol pairs appear elsewhere in $z$.

We summarize this shared-prefix information with a longest-common-prefix (LCP) array. For each suffix after the first in sorted order, its LCP value is the number of initial symbols it shares with the preceding suffix. A block of $r$ consecutive suffixes for which the $r-1$ intervening LCP values are all at least $h$ represents $r$ occurrences of the same length-$h$ substring. We enumerate the maximal such blocks at the prefix lengths present in the LCP array; these are the \emph{LCP intervals} in \cref{alg:suffix_array_candidates}. An interval is only a quick proposal, since its occurrences may overlap and replacing overlapping copies is impossible. We rank an interval containing $r$ suffixes with shared-prefix length $h$ by $h(r-1)$, which approximates the number of repeated symbols beyond the first copy. This ranking is used only to decide which proposals receive further evaluation, not to accept a grammar rule.

For each retained interval, we truncate its shared prefix to the maximum permitted rule length and use that prefix as the candidate pattern $q$. We then search the entire residual for every exact occurrence of $q$, rather than using only the starting positions in the interval. Scanning from left to right, we keep an occurrence whenever it begins after the preceding retained occurrence ends; because all occurrences have equal length, this produces a maximum-cardinality non-overlapping set. Patterns with fewer than two retained occurrences are discarded, and duplicate patterns proposed by different intervals are evaluated only once. Every surviving suffix-array pattern is finally scored with the same exact MDL gain $\Delta(q)$ as every pair-derived pattern.

\begin{algorithm}[t]
    \caption{MDL-guided maximal-repeat grammar induction}
    \label{alg:mdl_grammar_induction}
    \KwIn{String $x$; suffix-search interval $\tau$; suffix parameters $\theta_{\mathrm{SA}}=(\ell_{\min},\ell_{\max},K_{\mathrm{SA}},\rho,I_0)$}
    $G\gets\varnothing$, $z\gets x$, and $r\gets0$\;
    \While{\textnormal{true}}{
        $\mathcal{C}\gets$ maximal-repeat candidates obtained from every repeated pair\;
        \If{$K_{\mathrm{SA}}>0$ and either $\mathcal{C}$ has no positive-gain candidate or $[\tau>0$ and $(r+1)\bmod\tau=0]$}{
            $\mathcal{C}\gets\mathcal{C}\cup\operatorname{SuffixCandidates}(z;\theta_{\mathrm{SA}})$ using \cref{alg:suffix_array_candidates}\;
        }
        \If{$\mathcal{C}=\varnothing$}{
            \textbf{break}\;
        }
        $q^\star\gets\arg\max_{q\in\mathcal{C}}\Delta(q)$\;
        \If{$\Delta(q^\star)\leq 0$}{
            \textbf{break}\;
        }
        Add $A\to q^\star$ to $G$ and replace the non-overlapping occurrences of $q^\star$ in $z$ by $A$\;
        $r\gets r+1$\;
    }
    \textbf{return} $G$ and $z$\;
\end{algorithm}

\begin{algorithm}[t]
    \caption{Suffix-array repeat candidates}
    \label{alg:suffix_array_candidates}
    \KwIn{Residual $z$; length bounds $\ell_{\min},\ell_{\max}$; base interval budget $K_{\mathrm{SA}}$; oversampling factor $\rho$; minimum interval pool $I_0$}
    \KwOut{A set of repeated-pattern candidates $\mathcal{C}_{\mathrm{SA}}$}
    $\mathrm{SA}\gets$ starting positions of all suffixes of $z$, sorted lexicographically\;
    $\mathrm{LCP}\gets$ longest-common-prefix lengths of adjacent suffixes in $\mathrm{SA}$\;
    $B\gets\max\{\rho K_{\mathrm{SA}},I_0\}$\;
    Enumerate maximal LCP intervals with common-prefix length $h\geq\ell_{\min}$ and rank each interval containing $r$ suffixes by $h(r-1)$\;
    $\mathcal{I}\gets$ the $B$ highest-ranked intervals\;
    $\mathcal{C}_{\mathrm{SA}}\gets\varnothing$\;
    \ForEach{$I\in\mathcal{I}$, in decreasing score order}{
        $m\gets\min\{h_I,\ell_{\max}\}$ and $q\gets$ the length-$m$ prefix represented by $I$\;
        \If{$q$ has not already been examined}{
            Find every exact occurrence of $q$ in $z$ and greedily retain a left-to-right non-overlapping subset\;
            \If{at least two non-overlapping occurrences remain}{
                Add $q$ and its occurrence spans to $\mathcal{C}_{\mathrm{SA}}$\;
            }
        }
    }
    \textbf{return} $\mathcal{C}_{\mathrm{SA}}$\;
\end{algorithm}

The interval pool in \cref{alg:suffix_array_candidates} is deliberately larger than $K_{\mathrm{SA}}$: different high-ranked intervals can collapse to the same pattern after truncation, and other patterns fail the non-overlap test. Accordingly, $K_{\mathrm{SA}}$ sets the scale of the search rather than imposing a hard cap on its output. The routine can return as many as $B$ distinct patterns when all retained intervals survive, and every returned pattern is evaluated exactly before a rule is accepted.

Constructing and sorting the suffix array, computing the LCP values, and matching the retained patterns against all of $z$ is more expensive than updating the local pair index. We therefore invoke the global search only on iterations selected by $\tau$, or whenever the pair-derived candidates have no positive gain. The scheduled calls allow global candidates to compete before the local search stalls, while the fallback call prevents the pair-extension blind spot from terminating induction prematurely.

For the reported experiment, we set $\tau=4$, $\ell_{\min}=2$, $\ell_{\max}=128$, $K_{\mathrm{SA}}=24$, $\rho=8$, and $I_0=16$. Ties between equal-gain candidates are resolved by the earliest occurrence in the residual.

\paragraph{Conditional model compression.}

At a binary join $x_v=[x_a,x_b]$, the conditional compressor receives the grammar $G$ fitted to the combined substring and the two grammars $G_a,G_b$ fitted independently to its parts. It does not receive their residual sequences: the target quantity is the cost of describing the model $m_{x_v}$ given the part models, not the cost of reconstructing $x_v$. Operationally, the compressor tries to describe each rule of $G$ as either a copy of, a modification of, or an addition to the rules already available in $G_a$ and $G_b$. A copied rule should cost zero, a modified rule should cost according to the changed symbols, and a genuinely new rule should be charged in full.

The first problem is determining which rules correspond. Rule names are arbitrary identifiers: for example, the two grammars $(A\to 01,B\to AA)$ and $(C\to 01,D\to CC)$ have the same structure even though their names differ. Directly comparing their serialized symbol sequences would incorrectly charge for this relabelling. We therefore seek a consistent relabelling of each part grammar, meaning that we both reorder its rule definitions and replace every reference to an old rule name by the corresponding new name. Only after this alignment do we compare rule contents.

To handle grammars with different numbers and lengths of rules, we place them in arrays with a common number of rows and columns and fill unused rows and positions with a distinguished padding symbol. When this common number of rule rows is at most $6$, we can try every possible relabelling of a part grammar and retain the one with the smallest symbolic difference from $G$. This exhaustive procedure compares the actual rules and therefore finds the minimum-cost relabelling under the difference measure defined below.

For larger grammars, trying every relabelling is infeasible because the number of permutations grows factorially. We instead give each rule a fixed-length numerical descriptor, or \emph{embedding}, constructed from its terminal symbols, their positions, and the rules that it references. These features depend on the rule's role in the grammar rather than its integer name. Rules with similar descriptors can then be treated as candidate counterparts.

Mechanistically, each rule is represented by a $64$-dimensional vector. Every terminal symbol is assigned a fixed random vector generated with seed $0$ and shared across all three grammars. At initialization, a reference to another rule contributes no content beyond its position, so a rule's descriptor summarizes its terminals and overall layout. We then refine all rule descriptors simultaneously for three rounds. In each round, every rule reference is replaced by the preceding-round descriptor of the referenced rule before the enclosing rule is summarized again. One round therefore incorporates the contents of directly referenced rules, and successive rounds incorporate progressively deeper nested structure. This procedure is analogous to Weisfeiler--Lehman refinement \citep{leman1968reduction}, and a consistent relabelling of rule names leaves the resulting descriptors unchanged.

Because a rule's right-hand side is an ordered sequence, we add a sinusoidal position vector to the vector at each symbol position. Before averaging the position vectors into one descriptor for the rule, we replace the vector $u_k$ at position $k$ by $0.25u_{k-1}+0.5u_k+0.25u_{k+1}$, repeating the endpoint vector when a neighbour lies beyond the rule boundary. This local smoothing makes the descriptor less sensitive to a one-position displacement caused by a small insertion or deletion. We compute the squared Euclidean distance between every rule descriptor in $G$ and every descriptor in a part grammar, then use the Hungarian assignment algorithm to find the one-to-one correspondence with minimum total descriptor distance. This is only a tractable heuristic for choosing a relabelling: the final conditional score is computed from the aligned symbolic rules, not from their embedding distances.

Let $\alpha_j$ be a right-hand side in $G$ and let $\widetilde{\alpha}_{ij}$ be the rule aligned to it from part grammar $G_i$. Once the alignment is fixed, we charge the same number of bits for every terminal or rule reference on a right-hand side:
\begin{align}
    b = \log_2\max\left\{1,M+|\Sigma_G|\right\} ,
\end{align}
where $M$ is the number of rules in $G$ and $\Sigma_G$ is the set of distinct terminal symbols appearing in it. Thus, $M+|\Sigma_G|$ is the number of symbol types that may need to be named in the target grammar. The name on the left-hand side of a rule is implicit in its aligned row, so we charge only for symbols on right-hand sides. To avoid treating a short insertion at the beginning or end as a change to every subsequent position, we compare each pair of rules under several relative shifts:
\begin{align}
    d(\alpha,\beta)
    =
    b\min_{\delta\in\{-S,\ldots,S\}}
    d_{\mathrm{H}}\!\left(\operatorname{pad}_S(\alpha),
    \operatorname{shift}_{\delta}(\operatorname{pad}_S(\beta))\right),
    \qquad S=10,
\end{align}
where padding adds empty positions at both ends, $\operatorname{shift}_{\delta}$ moves the second padded rule by $\delta$ positions, and the Hamming distance $d_{\mathrm H}$ simply counts positions containing different symbols. We select the cheapest shift from $-10$ to $10$ and multiply its mismatch count by $b$. Identical aligned rules therefore cost zero. Comparing a rule with an entirely empty row costs $b|\alpha|$, the same as writing every symbol of $\alpha$ directly. After aligning each part grammar separately, each rule of the combined grammar independently chooses among three descriptions: write it directly, describe its differences from the aligned rule in $G_a$, or describe its differences from the aligned rule in $G_b$:
\begin{align}
    L_{\mathrm{cond}}^g(G\mid G_a,G_b)
    =
    \sum_{j=1}^{M}
    \min\left\{
        b|\alpha_j|,
        d(\alpha_j,\widetilde{\alpha}_{aj}),
        d(\alpha_j,\widetilde{\alpha}_{bj})
    \right\}.
    \label{eq:conditional_grammar_code}
\end{align}

\begin{algorithm}[t]
    \caption{Conditional compression of a combined grammar}
    \label{alg:conditional_grammar_compression}
    \KwIn{Combined grammar $G$; part grammars $G_a,G_b$; exact-alignment threshold $M_{\mathrm{exact}}$}
    Pad all grammars to a common shape and compute $b$\;
    \ForEach{$i\in\{a,b\}$}{
        \eIf{the padded grammars have at most $M_{\mathrm{exact}}$ rules}{
            Find the minimum-difference relabelling of $G_i$ by exhaustive permutation\;
        }{
            Embed the rules, then find a one-to-one alignment using Hungarian assignment\;
        }
    }
    \ForEach{rule $\alpha_j$ in $G$}{
        $c_j\gets\min\{b|\alpha_j|,d(\alpha_j,\widetilde{\alpha}_{aj}),d(\alpha_j,\widetilde{\alpha}_{bj})\}$\;
    }
    \textbf{return} $\min\{b\sum_j|\alpha_j|,\sum_j c_j\}$\;
\end{algorithm}

This design targets three desired behaviours. First, exact alignment assigns zero cost to a grammar conditioned on an isomorphic relabelling of itself. Second, if the combined grammar is a union of rules inherited from the two parts, the per-rule minimum permits each rule to select the part from which it came and therefore permits a zero total cost. Third, small changes to an inherited rule are charged in proportion to their symbolic differences, while a genuinely unmatched rule pays its full raw cost. We use the difference-based strategy described above, with exhaustive alignment threshold $6$, maximum shift $10$, embedding dimension $64$, and three message-passing rounds.

This procedure is intentionally an operational approximation to conditional model complexity, not a complete lossless coding protocol. In particular, the inferred alignment is not transmitted separately, the rule boundaries are treated as known, and the fixed per-symbol cost is only a relative complexity scale. These choices remove arbitrary rule identities and focus the score on structural differences between grammars, which is the property needed by the reward.

\paragraph{Decomposition and reward.}

The balanced decomposition contains $32$, $16$, $8$, $4$, $2$, and $1$ substrings of lengths $4$, $8$, $16$, $32$, $64$, and $128$, respectively. Evaluating one candidate therefore fits $63$ grammars and computes the conditional cost at each of the $31$ binary joins. Conditional costs are divided by the length of the combined substring, averaged over the $16$, $8$, $4$, $2$, and $1$ joins at their respective scales, and aggregated with the $p=0.75$ power mean given in the main text.

\paragraph{Optimization.}

We implemented the evolutionary search with \texttt{pymoo} \citep{pymoo}, using NSGA-II \citep{deb2002fast} with a single objective equal to the negative reward. Although NSGA-II is formulated for multi-objective optimization, in this setting its tournament selection and parent--offspring survival reduce to an elitist genetic search over a scalar score; we refer readers to the original work for details of ranking and survival.

The initial population contains $10{,}000$ strings sampled independently and uniformly from the eight-symbol vocabulary. Each generation produces $5{,}000$ offspring. Selected parent pairs undergo two-point crossover with probability $0.9$. Mutation is then applied to every offspring: each of its $128$ positions is independently selected with probability $2/128$, and a selected position is reset to a uniformly sampled vocabulary symbol, possibly the symbol already present. Thus, each offspring has two attempted symbol resets in expectation. Duplicate strings are eliminated. Parents and offspring are scored, and NSGA-II retains an elite population of $10{,}000$ strings for the next generation.

\paragraph{Additional results.}

\cref{fig:generated_string_dot_all} shows the dot plot of the top-reward string for all $10$ seeds after optimizing both the compositionality and structure rewards. Methods for dot plot generation and its interpretation are described in the main text.

\begin{figure}[ht]
    \centering
    \includegraphics[width=\linewidth]{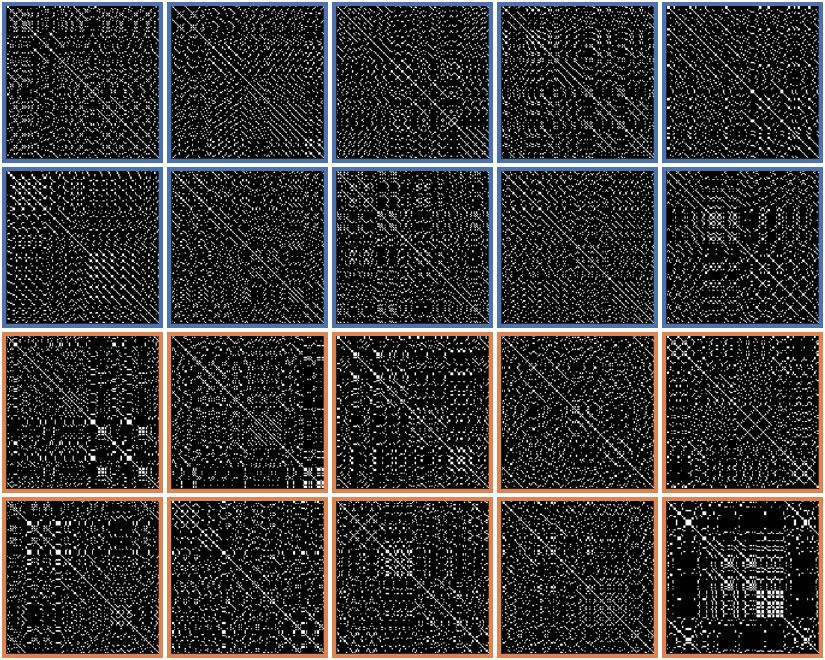}
    \caption{\textbf{Dot plots for all seeds.} Compositionality-optimized strings have \textcolor{MidnightBlue}{\textbf{blue}} borders, and structure-optimized strings have \textcolor{BurntOrange!50!RawSienna}{\textbf{orange}} borders.}
    \label{fig:generated_string_dot_all}
\end{figure}

\subsection{Compositional images}
\label{sec:experiment_details_images}

\paragraph{Coordinate-network model.}

For every patch $x_v\in\{0,\ldots,3\}^{H\times W\times 3}$, the two-part compressor fits a separate probabilistic coordinate network. Let $u_{ij}\in[-1,1]^2$ be the coordinate of pixel $(i,j)$, normalized within the patch so that its two axes span $[-1,1]$. The network outputs $12$ real-valued scores, or logits, arranged as four logits for each of the three channels, and thereby defines a categorical distribution $q_{\theta,c}(\cdot\mid u_{ij})$ for channel $c$. Its likelihood factorizes over coordinates and channels:
\begin{align}
    q_\theta(x_v)
    =\prod_{i=1}^{H}\prod_{j=1}^{W}\prod_{c=1}^{3}
    q_{\theta,c}(x_{v,ijc}\mid u_{ij}).
    \label{eq:cppn_patch_likelihood}
\end{align}
This factorization does not make pixels independent in the model's description: the same parameters determine the distribution at every coordinate, so spatially related pixels can share structure through the network.

The network has three hidden layers. Each layer contains four groups of $32$ units, to which we respectively apply the functions $z$, $\sigma(5z)$, $\tanh(2.5z)$, and $\exp(-5z^2)$. Concatenating the four groups gives a hidden width of $128$. The mixture of monotone, saturating, and localized nonlinearities is standard for CPPNs and gives a small fixed architecture several biases for representing smooth regions, boundaries, symmetries, and repeated spatial patterns.

Alongside this spatial model, we fit a coordinate-independent baseline $q_0$. For each channel $c$, $q_{0,c}$ is its empirical categorical distribution in the training pixels, with an additive pseudocount of $10^{-3}$ for each intensity. This baseline can encode colour frequencies but no relationship between a value and its position. We assign it zero model complexity, making it the empty-model analogue against which learned spatial structure must justify itself.

\paragraph{Training and model selection.}

We randomly divide the spatial coordinates of each patch equally into training and validation sets. The coordinate network is trained only on the training half by minimizing mean categorical cross-entropy over coordinates and channels. We use full-batch AdamW with initial learning rate $4\times10^{-3}$, weight decay $10^{-4}$, gradient-norm clipping at $1$, and cosine annealing to $4\times10^{-4}$. Training runs for at most $250$ gradient steps. We evaluate after the first step and every five steps thereafter, retain the checkpoint with the lowest validation cross-entropy, and stop after at least $10$ steps once eight consecutive evaluations fail to improve validation loss by at least $10^{-4}$ bits per symbol.

Held-out selection serves two roles. First, early stopping restricts how much pixel-level detail the network can memorize. Second, we compare the selected checkpoint with the constant baseline on validation coordinates. We retain the coordinate network only if it reduces validation loss by at least $0.02$ bits per symbol and by at least $2\%$ relative to the baseline; otherwise, the empirical model $m_{x_v}^f$ is the constant model. These margins prevent random fluctuations from turning weak coordinate dependence into apparent structure. Network initialization and the train--validation division are resampled when a patch model is fitted.

\paragraph{Two-part complexity estimates.}

Let $\ell_r$ be the training cross-entropy in bits per pixel-channel at evaluation $r$, and let $s_r$ be the cumulative number of training coordinates processed by that point. If $r^\star$ is the validation-selected checkpoint, write $\ell_\star=\ell_{r^\star}$ and $s_\star=s_{r^\star}$. We estimate the model-description term by the trapezoidal approximation to:
\begin{align}
    \widehat L_{\mathrm{model}}^f(q_\theta)
    =3\int_0^{s_\star}\left[\ell(s)-\ell_\star\right]_+\,\mathrm ds,
    \label{eq:cppn_model_complexity}
\end{align}
where $\ell(s)$ linearly interpolates the recorded training losses and $[a]_+=\max\{a,0\}$. The factor $3$ converts coordinates seen into channel symbols seen. Intuitively, $\ell(s)-\ell_\star$ removes the uncertainty that remains even after the selected model has been learned; the remaining excess is the predictive cost paid because its structure has not yet been absorbed by the weights. Integrating this excess over learning gives a cheap learning-curve proxy for the information acquired by the network. This follows the same principle as online, or prequential, estimates of neural-network description length, which accumulate predictive codelength as training examples become available sequentially \citep{blier2018description}, and the computationally inexpensive learning-complexity proxy used by Epiplexity \citep{finzi2026epiplexity}.

The model-to-data term is the selected model's negative log-likelihood on the complete patch, including both the training and held-out coordinates:
\begin{align}
    \widehat L_{\mathrm{data}}^f(x_v\mid q)
    =-\sum_{i,j,c}\log_2 q_c(x_{v,ijc}\mid u_{ij}).
    \label{eq:cppn_data_complexity}
\end{align}
If the constant baseline is selected, we set $\widehat L_{\mathrm{model}}^f(q_0)=0$ and use its full-patch negative log-likelihood in \cref{eq:cppn_data_complexity}. Otherwise, the two-part length is the sum of \cref{eq:cppn_model_complexity,eq:cppn_data_complexity}.

Unlike a true prequential/online code, \cref{eq:cppn_model_complexity} does not directly estimate the actual codelength of the final weights. We use it as a cheap and consistently applied surrogate expected to increase with the amount of learnable structure stored in the network. The validation split, constant fallback, and residual likelihood are therefore essential complements: they prevent a network from receiving structural credit merely for taking a long time to memorize noise. \cref{alg:cppn_two_part_code} summarizes the complete procedure.

\begin{algorithm}[t]
    \caption{Neural two-part compression of an image patch}
    \label{alg:cppn_two_part_code}
    \KwIn{Patch $x_v$ and normalized coordinates $\{u_{ij}\}$}
    Randomly split the coordinates equally into training set $I_{\mathrm{tr}}$ and validation set $I_{\mathrm{val}}$\;
    Fit the coordinate-independent channel marginals $q_0$ on $I_{\mathrm{tr}}$\;
    Randomly initialize the fixed coordinate network $q_\theta$\;
    \For{$t=1,\ldots,t_\text{max}$}{
        Update $\theta$ on the full training set using AdamW\;
        At scheduled evaluations, record $(s_t,\ell_t)$ and the validation loss, and save the best checkpoint\;
        Stop if the validation patience criterion is met\;
    }
    Restore the validation-selected checkpoint $q_{\theta^\star}$\;
    \eIf{$q_{\theta^\star}$ improves on $q_0$ by both validation margins}{
        $q\gets q_{\theta^\star}$ and compute $\widehat L_{\mathrm{model}}^f(q)$ using \cref{eq:cppn_model_complexity}\;
    }{
        $q\gets q_0$ and $\widehat L_{\mathrm{model}}^f(q)\gets0$\;
    }
    Compute $\widehat L_{\mathrm{data}}^f(x_v\mid q)$ on every coordinate using \cref{eq:cppn_data_complexity}\;
    \textbf{return} $m_{x_v}^f=q$, $\widehat L_{\mathrm{model}}^f(q)$, and $\widehat L_{\mathrm{data}}^f(x_v\mid q)$\;
\end{algorithm}

\paragraph{Conditional model compression.}

At each quadtree join, the target model is the network $q_v$ fitted to the parent patch, while the four given models $q_1,\ldots,q_4$ were fitted independently to its quadrants. Direct parameter comparison would be arbitrary: hidden units can be permuted, rescaled, or otherwise reparameterized without changing the represented function. We instead ask how much function-level adaptation is required to turn the children into the parent.

We first construct an unadapted \emph{routed student}. Every coordinate in the parent is assigned to its quadrant $a(i,j)\in\{1,\ldots,4\}$ and converted to that quadrant's local normalized coordinates $\widetilde u_{ij}$. The student's prediction is:
\begin{align}
    q_S^{(0)}(\cdot\mid u_{ij})
    =q_{a(i,j)}(\cdot\mid\widetilde u_{ij}).
    \label{eq:routed_student}
\end{align}
Thus, $q_S^{(0)}$ represents the null hypothesis that the model of the parent is only a spatial juxtaposition of the four child models. No new routing parameters are learned or charged: the quadrant routing is fixed by the decomposition tree.

The parent acts as a teacher. We cache its categorical distributions at every parent coordinate, copy the four child models, and fine-tune all copied coordinate networks to minimize the forward divergence:
\begin{align}
    d(q_v,q_S;I)
    =\frac{1}{3|I|}\sum_{(i,j)\in I}\sum_{c=1}^{3}
    D_{\mathrm{KL}}\!\left(q_{v,c}(\cdot\mid u_{ij})\,\|\,
    q_{S,c}(\cdot\mid u_{ij})\right)
    \label{eq:cppn_distillation_loss}
\end{align}
in bits per symbol. A child that fell back to the constant model has no trainable parameters and remains fixed; all other child weights are adapted. We again use a random $50\%$ training--validation split, the optimizer and stopping schedule above, and validation checkpointing. The adapted student is retained only if its validation divergence is lower than that of the unadapted routed student by at least $0.02$ bits per symbol and by at least $2\%$. Otherwise, we revert all four components to their original child models.

The conditional model term uses the same learning-curve construction as \cref{eq:cppn_model_complexity}, with the training cross-entropy $\ell$ replaced by the training distillation loss $d$ and its selected value $d_\star$:
\begin{align}
    \widehat L_{\mathrm{model,cond}}^g
    =3\int_0^{s_\star}[d(s)-d_\star]_+\,\mathrm ds.
    \label{eq:cppn_conditional_model_complexity}
\end{align}
If adaptation is rejected, this term is zero. The conditional residual is the remaining teacher--student divergence over every coordinate and channel:
\begin{align}
    \widehat L_{\mathrm{resid,cond}}^g
    =\sum_{i,j,c}D_{\mathrm{KL}}\!\left(q_{v,c}(\cdot\mid u_{ij})\,\|\,
    q_{S,c}(\cdot\mid u_{ij})\right),
    \label{eq:cppn_conditional_residual}
\end{align}
and our estimate of the conditional model complexity is their sum:
\begin{align}
    L_{\mathrm{cond}}^g(m_{x_v}^f\mid m_{x_{v_1}}^f,\ldots,m_{x_{v_4}}^f)
    =\widehat L_{\mathrm{model,cond}}^g+\widehat L_{\mathrm{resid,cond}}^g.
    \label{eq:cppn_conditional_total}
\end{align}
This construction gives the desired limiting behaviour. If the parent model is reproduced by simply routing among the children, no adaptation is needed and the residual divergence is zero. If modeling the quadrants together changes the useful spatial structure, reproducing the parent requires either a nontrivial learning trajectory, a remaining functional correction, or both. As with the two-part model term, \cref{eq:cppn_conditional_model_complexity} is a surrogate rather than a complete bitstream; the purpose of the residual is to ensure that cheap adaptation is not mistaken for exact recovery. The full procedure is given in \cref{alg:cppn_conditional_code}.

\begin{algorithm}[t]
    \caption{Conditional compression of a parent coordinate model}
    \label{alg:cppn_conditional_code}
    \KwIn{Parent model $q_v$; quadrant models $q_1,\ldots,q_4$; parent coordinate grid}
    Cache the teacher distributions produced by $q_v$ on the parent grid\;
    Construct the unadapted routed student $q_S^{(0)}$ using \cref{eq:routed_student}\;
    Randomly split parent coordinates equally into training set $I_{\mathrm{tr}}$ and validation set $I_{\mathrm{val}}$\;
    Copy the child models to obtain $q_S$\;
    \For{$t=1,\ldots,t_\text{max}$}{
        Fine-tune the trainable child copies on $I_{\mathrm{tr}}$ using \cref{eq:cppn_distillation_loss}\;
        At scheduled evaluations, record $(s_t,d_t)$ and validation divergence, and save the best checkpoint\;
        Stop if the validation patience criterion is met\;
    }
    Restore the validation-selected student $q_S^\star$\;
    \eIf{$q_S^\star$ improves on $q_S^{(0)}$ by both validation margins}{
        $q_S\gets q_S^\star$ and compute $\widehat L_{\mathrm{model,cond}}^g$ using \cref{eq:cppn_conditional_model_complexity}\;
    }{
        $q_S\gets q_S^{(0)}$ and $\widehat L_{\mathrm{model,cond}}^g\gets0$\;
    }
    Compute $\widehat L_{\mathrm{resid,cond}}^g$ on the complete parent grid using \cref{eq:cppn_conditional_residual}\;
    \textbf{return} $\widehat L_{\mathrm{model,cond}}^g+\widehat L_{\mathrm{resid,cond}}^g$\;
\end{algorithm}

\paragraph{Decomposition and reward.}

The $64\times64$ image yields $64$ patches of size $8\times8$, $16$ patches of size $16\times16$, four patches of size $32\times32$, and one complete-image patch. We therefore fit $85$ two-part models and evaluate $21$ four-child joins per image. For a join with an $H\times W$ parent, we divide \cref{eq:cppn_conditional_total} by $HW$, giving conditional complexity per spatial location. Let $r_{16}$, $r_{32}$, and $r_{64}$ be the mean normalized conditional complexities among joins whose parent widths are $16$, $32$, and $64$, respectively. The scalar fitness is the power mean:
\begin{align}
    R(x)=\left(\frac{r_{16}^{p}+r_{32}^{p}+r_{64}^{p}}{3}\right)^{1/p},
    \qquad p=0.75.
    \label{eq:image_reward}
\end{align}
Averaging within a scale prevents the more numerous small patches from dominating, while the power mean gives the three scales equal overall status and is more sensitive than the arithmetic mean to weaker scales.

\paragraph{Evolutionary optimization.}

The coordinate networks above are fitted probabilistic models used only to evaluate the reward. The CPPNs used to generate each candidate image are separate networks whose architectures and parameters are evolved. Each receives the full-image coordinate $(u_x,u_y)\in[-1,1]^2$ and produces three scalar outputs $y_c\in[-1,1]$. We convert each output to an intensity by:
\begin{align}
    x_{ijc}=\operatorname{round}\!\left(\frac{(y_c(u_{ij})+1)(4-1)}{2}\right),
\end{align}
so a single deterministic CPPN genome specifies all $64\times64\times3$ values.

We evolve these genomes with feed-forward NEAT \citep{stanley2002evolving,stanley2004competitive}, implemented using \texttt{neat-python} version $2.1.0$ \citep{neat-python}. NEAT begins with simple networks, introduces nodes and connections through mutation, crosses over related genomes, and groups similar genomes into species so that newly complexified structures can compete primarily with close relatives. We defer the precise crossover, speciation, and offspring-allocation rules to the original papers. Our settings, summarized in \cref{tab:image_neat_hyperparameters}, are close to those used for CPPN evolution in Picbreeder \citep{secretan2008picbreeder,secretan2011picbreeder}.

\begin{table}[t]
    \centering
    \small
    \caption{NEAT hyperparameters for the image experiment. Mutation probabilities are applied by \texttt{neat-python} during genome reproduction.}
    \label{tab:image_neat_hyperparameters}
    \begin{tabular}{p{0.27\linewidth}p{0.65\linewidth}}
    \toprule
    Setting & Value \\
    \midrule
    Population and initialization & $200$ feed-forward genomes; initially no hidden nodes, with both coordinate inputs fully connected to all three outputs \\
    Activations & Hidden-node default tanh; mutation options identity, sigmoid, tanh, and Gaussian; activation mutation probability $0.12$; output activations fixed to tanh \\
    Structural mutation & Add connection $0.35$; delete connection $0.08$; add node $0.25$; delete node $0.05$ \\
    Weights and biases & Gaussian initialization with mean $0$ and standard deviation $1$; weight mutation probability $0.8$; bias mutation probability $0.7$; replacement probability $0.1$; additive-mutation standard deviation $0.75$; values clipped to $[-30,30]$ \\
    Speciation & Compatibility threshold $2.8$; coefficient $1.0$ for nodes or connections present in only one genome; coefficient $0.8$ for weight differences \\
    Stagnation & Species considered stagnant after $30$ generations without improvement; the three highest-fitness species protected from stagnation removal \\
    Reproduction & Two top genomes copied unchanged per species; top $25\%$ of each species eligible to reproduce; minimum species size $2$ \\
    \bottomrule
    \end{tabular}
\end{table}

\paragraph{Image-generating CPPN network statistics.}

To quantitatively analyze the structure of the evolved images, we compute summary statistics of the CPPN networks that generated them, summarized in \cref{tab:cppn_statistics}.

\begin{table}[t]
    \centering
    \small
    \caption{CPPN structural statistics recorded during evolution. Each metric is summarized across the population using its mean and maximum.}
    \label{tab:cppn_statistics}
    \begin{tabular}{p{0.31\linewidth}p{0.61\linewidth}}
    \toprule
    Metric & Description \\
    \midrule
    Nodes
    & Number of hidden nodes. \\
    Connections
    & Number of connections between nodes. \\
    Depth
    & Longest input-to-output path length, measured in connections. \\
    $\log_{10}(1 + \#\text{Paths})$
    & Number of distinct input-to-output paths. \\
    Fan-out (max)
    & Maximum number of connections going out of a hidden node. \\
    Fan-in (max)
    & Maximum number of connections going in to a hidden node. \\
    Reused nodes (\%)
    & Fraction of hidden nodes with fan-out greater than one. \\
    Combining nodes (\%)
    & Fraction of hidden nodes with fan-in greater than one. \\
    \bottomrule
    \end{tabular}
\end{table}

\paragraph{Additional example images.}

\cref{fig:generated_images_hierarchical_all} shows additional generated images using compositionality reward, and \cref{fig:generated_images_sophistication_all} for the structure reward.

\begin{figure}[ht]
    \centering
    \includegraphics[width=\linewidth]{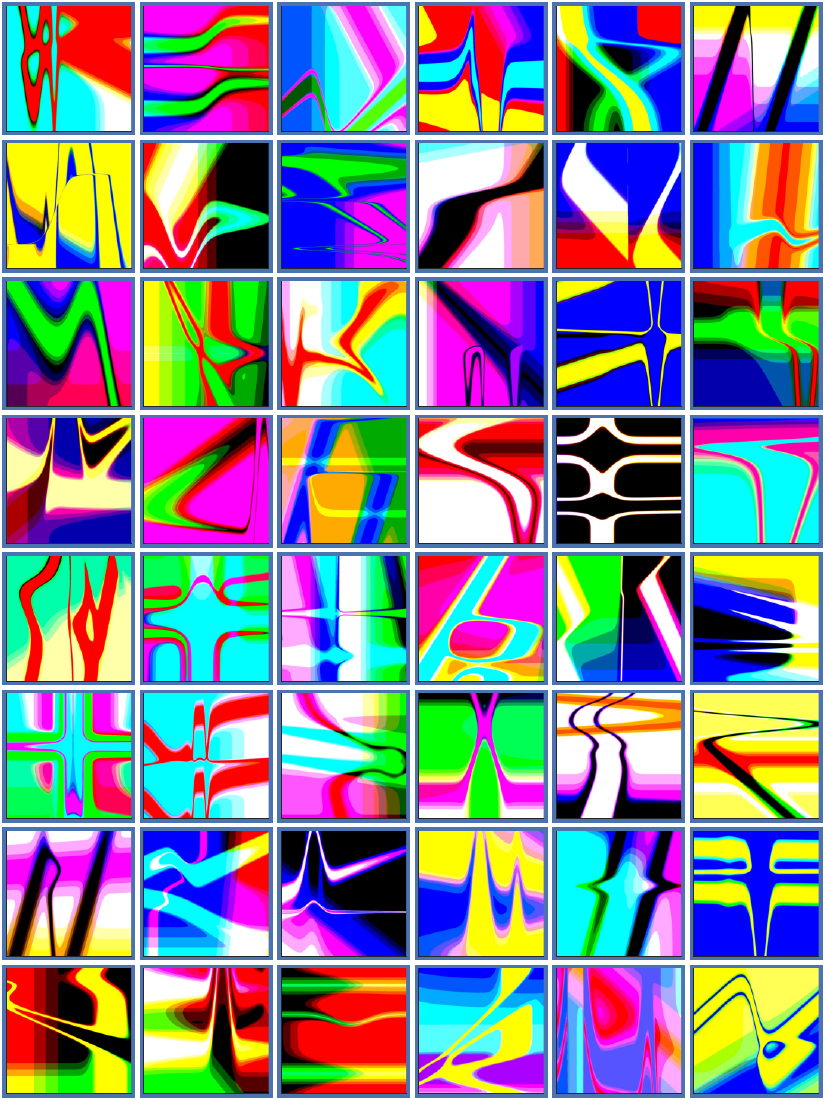}
    \caption{\textbf{Additional images generated from compositionality-optimized CPPNs.}}
    \label{fig:generated_images_hierarchical_all}
\end{figure}
\begin{figure}[ht]
    \centering
    \includegraphics[width=\linewidth]{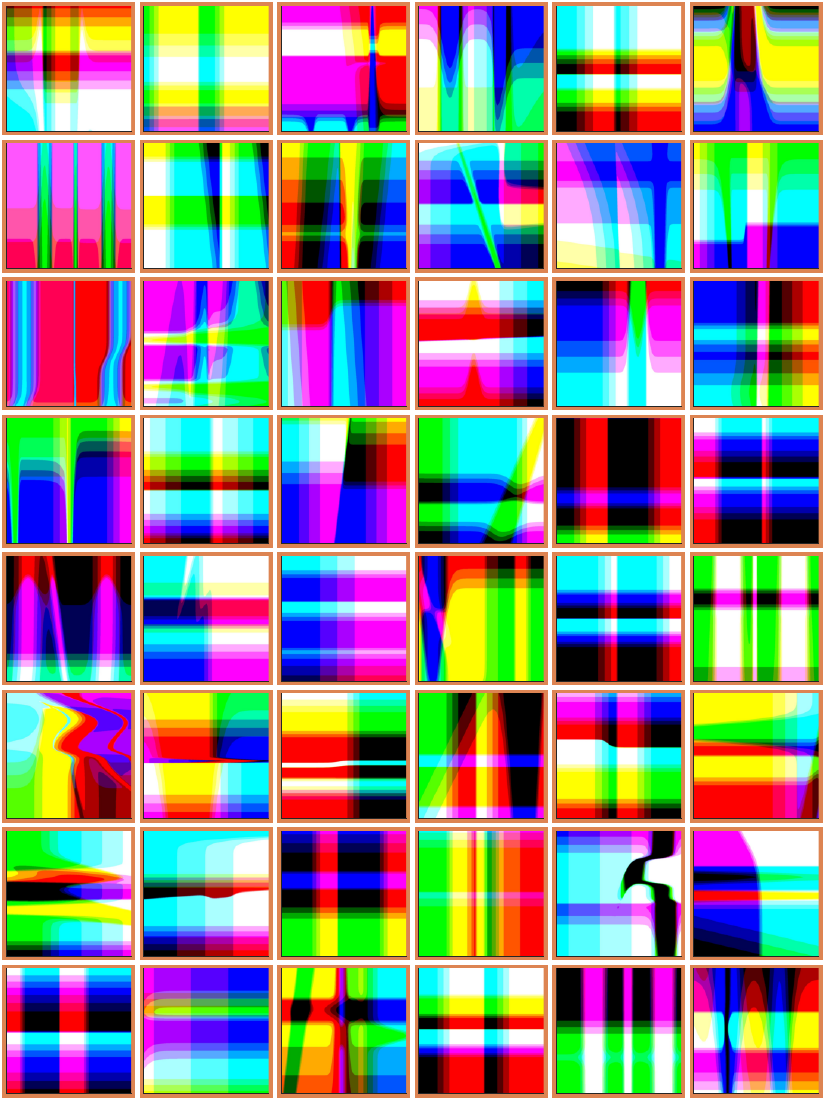}
    \caption{\textbf{Additional images generated from structure-optimized CPPNs.}}
    \label{fig:generated_images_sophistication_all}
\end{figure}

\end{appendices}

\end{document}